\documentclass[times,review,10pt]{elsarticle}
\usepackage[a4paper, total={11.4cm, 21.1cm}]{geometry}
\usepackage{amssymb}
\usepackage{amsthm}
\usepackage{color}
\usepackage{amsmath}
\usepackage{enumitem}
\usepackage{graphicx}
\usepackage{caption}
\usepackage{booktabs}
\usepackage{subcaption}
\usepackage[T1]{fontenc}
\usepackage[utf8]{inputenc}
\usepackage{afterpage}
\usepackage[numbers]{natbib}
\usepackage{multirow}
\usepackage[hyphens,spaces,obeyspaces]{url}
\usepackage[hidelinks]{hyperref}
\hypersetup{colorlinks=true,linkcolor=blue,urlcolor=blue}
\usepackage{algorithm}
\usepackage{algpseudocode}
\usepackage{comment}

\theoremstyle{plain}
\newtheorem{theorem}{Theorem}

\newtheorem{definition}[theorem]{Definition}

\newtheorem{lemma}[theorem]{Lemma}

\newtheorem{remark}[theorem]{Remark}

\usepackage{amssymb}
 \usepackage[table,xcdraw]{xcolor}
\usepackage{pdflscape}

\makeatletter
\def\ps@pprintTitle{%
 \let\@oddhead\@empty
 \let\@evenhead\@empty
 \def\@oddfoot{\reset@font\hfil}%
 \let\@evenfoot\@oddfoot}
\makeatother

\begin{document}

\begin{frontmatter}

\title{On Generalized Naive Bayes with Continuous Features}

\author[inst1]{Ábrahám Papp}
\author[inst2]{Botond Szilágyi}
\author[inst1]{Edith Alice Kovács}

\address[inst1]{Department of Analysis and Operations Research, Institute of Mathematics, Budapest University of Technology and Economics}

\address[inst2]{Department of Chemical and Environmental Process Engineering, Faculty of Chemical Technology and Biotechnology, Budapest University of Technology and Economics}

\begin{abstract}
The  Generalized Naive Bayes (GNB) model was introduced for discrete and categorical random variables as an extension of classic Naive Bayes.  We now accommodate the GNB framework to continuous explanatory variables. A central result of the paper is that structure learning of the GNB  depends only on the pair copulas of the bi-variate marginals. 
We proved that the GNB structure can be assigned to the basis of a matroid, therefore we give greedy algorithms for finding the optimal GNB structure on the training data, in sense of minimizing Kullback-Leibler divergence. Three cases are considered: joint Gaussian distribution, then a more flexible model where we suppose the dependence structure to be described by a Gaussian copula with arbitrary marginals, and an even more flexible case where the joint continuous probability distribution is arbitrary, i.e. copula and marginal distributions are arbitrary. A method for model reduction, based on the newly introduced concept of GNB forest is given.
We close the paper by comparing the newly introduced GNB classification results to other classical "glass-box" algorithms on real datasets.
\end{abstract}

\begin{highlights}
\item The Generalized Naive Bayes (GNB) structure is extended to the case of continuous explanatory variables.
\item It is proven that finding the GNB structure, depends only on the joint copula of the feature pairs.
\item We discuss the combinatorial properties of the GNB structures.
\item Greedy algorithms for GNB structure discovering under three dependence scenarios are given.
\item Classification algorithms are given using the whole and reduced GNB structure.
\end{highlights}


\end{frontmatter}

\section{Introduction}
Naive Bayes classifier is one of the most popular machine learning algorithms due to its simplicity, efficiency and easy interpretation, which is appealing to experts in various domains. Therefore, it is considered to be one of the top 10 data mining algorithms \cite{zhang2021attribute}, \cite{wu2008top}. However, most real-world domains involve continuous variables. A common practice to deal with continuous variables is to discretize them, with a subsequent loss of information. In this paper we are especially interested in cases when the explanatory variables are continuous.

To illustrate the ubiquitous interest in this algorithm we provide a glimpse of some recent applications of it.
Gaussian Naive  Bayes classification is applied for detector pulse discrimination, see \cite{petschke2019supervised}. In \cite{aguilar2025xnb} a new class-specific feature selection is proposed, instead of using the same features for all classes. For probability estimation authors use kernel density estimation. A new efficient classification weighted Naive Bayes classifier was introduced in \cite{ou2025novel}. The probabilities were weighted based on a mean of three weights mutual information, chi-square (independence test) and Gini impurity calculated on discretized data. We consider very important to mention here the study \cite{brandao2025optimization} where  an approach for building optimized machine learning models for the classification of opinions on social media posts is proposed. Their experiments shown that although RNN architectures are widely recognized for their ability to handle sequential data, classical algorithms such as Naive Bayes and SVM still offer comparable performance, especially when preprocessing techniques were used. They showed that Naive Bayes outperformed RNN on most datasets. 

An emerging problem today is the management of decentralized data sources while preserving data privacy. The paper \cite{torrijos2024federated} proposes a new federated approach for Naive Bayes (NB) classification, assuming discrete variables. Their approach federates a discriminative variant of NB, sharing meaningless parameters instead of conditional probability tables. Therefore, this process is more reliable against possible attacks. 
 
\vspace{3mm}
 
 The NB algorithm is based on the assumption that the explanatory variables are conditionally independent given the class variable (target variable). This is not generally true for real-life problems.
Many papers are concerned with the improvement of this remarkable algorithm. These improvements follow mainly two directions. Some research show that selecting some attributes beforehand might result in better classification accuracy and better generalization. Another improvement might be achieved by relaxing the conditional independence assumption.
Essential results in this direction, using continuous predictors without discretizations, are introduced in the following papers. In \cite{friedman1998bayesian} a former method, called Tree Augmented Naive Bayes (TAN) \cite{friedman1997bayesian} which was defined for discrete attributes, was extended to continuous attributes having Gaussian or mixture or Gaussian distributions. Furthermore, they proposed a hybrid model that can simultaneously accommodate both continuous and discrete variables. In \cite{perez2006supervised}, the application of conditional Gaussian distributions was further investigated and adapted to this framework. The study also introduced the $k$-Dependence Bayesian classifier, originally proposed for the discrete setting in \cite{sahami1996learning}. In Sahami's paper \cite{sahami1996learning}, the $k$-dependence classifier includes the structure of a Naive Bayes classifier and allows each feature to have at most $k$ parent features.
 In \cite{perez2006supervised} the authors improved the former method and called it filter $k$-dependence Bayes classifier.
In \cite{perez2009bayesian}, the authors examined the classification framework from the perspective of Bayesian networks, which represent conditional independence relationships among random variables using directed acyclic graphs. They proposed several flexible classifiers based on kernel density estimation, including the Flexible Naive Bayes, Flexible Tree-Augmented Naive Bayes, and Flexible $k$-Dependence Bayesian classifiers (FKDB). In the FKDB model, each predictor variable is allowed to have at most $k$ parent predictors in addition to the class variable, thereby providing greater modeling flexibility while controlling structural complexity. Conditional mutual information between the features with respect to the classifier is considered in their greedy algorithms.
A novel selective Naive Bayes algorithm was introduced in \cite{chen2020novel}, where after ordering the features based on the mutual information between the features and target variable, each new model is obtained from the previous one by adding one feature greedily. The "nest" model is chosen based on its predictability power. In \cite{blanquero2021variable} authors gave a new version of Naive Bayes classifier dealing with datasets with correlated patterns. The method involves reducing variables through clustering, based on their dependencies. The goal was to select a combination of features that are as independent as possible.

There are also some recent papers dealing with $k$-dependence classifiers. In \cite{wang2022semi} for each node up to $k$ parents are added in greedy way taking into account the conditional log-likelihood. Another paper concerning with the introduction of higher order dependencies into a Bayesian network classifier is \cite{wang2024learning}. Their idea is to start from a TAN structure and to add directed edges based on conditional entropy.  Although increasing the value of $k$ may raise the risk of  overfitting, fixing $k$ may constrain the model flexibility. To address these issues in \cite{meng2026k}, the authors propose $K$-free dependence Bayesian classifiers, which can learn an adaptive number of parent nodes for each attribute. To search its optimal structure, they sequentially evaluate the candidate submodels either by minimizing the mean squared error or maximizing the classification accuracy.

A very interesting new paper \cite{wang2025improving} improves the Gaussian Naive Bayes classifier on imbalanced datasets through coordinate-based minority feature mining, using coordinate transformation based on local relative density changes. In \cite{ou2025relaxed} a relaxed Naive Bayesian classifier was introduced based on the maximum dependent attribute groups, is proposed to alleviate the conditional independence assumption and then to obtain better prediction performance than traditional Naive Bayes and its most of variants.

Ensemble learning can also lead to performance improvement. In \cite{ren2022stochastic} random Bayes forest is proposed by introduction of randomness (via a probability distribution on the mutual information) in root selection, children selection and parent selection. The classification is delivered by voting.

A comparison of the three main methods for handling continuous variables: the normal method, the kernel method, and discretization is given in \cite{bouckaert2004naive}. They showed that no single method is universally better than the other. The authors present three model-selection methods based on cross-validation that can significantly improve overall performance of Naive Bayes classifiers, outperforming any of the three popular methods on their own. 

The present paper aims to relax the conditional independence assumption by introducing a structure that allows dependencies among the explanatory attributes. We create the theoretical background for the case of continuous attributes, then we extend the graphical structure allowing more complex dependencies. We discuss on the optimality of the greedy algorithms introduced. We also compare classification results to other "glass-box" methods like support vector machine (SVM) and logistic regression (LR).
We mention here the paper \cite{ng2001discriminative} of Ng et.al.  which compares Naive Bayes algorithm with logistic regression.  They claimed that there exist two distinct regimes of performance between the generative and discriminative classifiers with regard to the training-set size. Another paper \cite{xue2008comment}
suggests through simulation that
the existence of the two distinct regimes may not be so reliable. In addition, for real world datasets, so far there is no theoretically correct, general criterion for choosing between the
discriminative and the generative approaches for classification. The choice depends on the confidence we have in the correctness of the specification of either $p(y|\mathbf{x})$ or $p(\mathbf{x}, y)$.

\vspace{2mm}

The remainder of this paper is organized as follows. In Section~\ref{sec:preliminaries}, key concepts and main results from probability theory, information theory and graph theory needed are introduced.
Section~\ref{sec:gnb_weight} introduces the weights assigned to the approximation under different assumptions on the joint distribution of the explanatory variables. This section contains the proof that the structure of GNB depends only on the copula of the features, which gives an absolutely new perspective to this field.

In Section~\ref{sec:combinatorial background}, we discuss the combinatorial background of the GNB structures and give greedy algorithms which guarantee to find the optimal structure on the training data in the sense of minimizing Kullback-Leibler divergence under different dependence assumptions. Moreover we introduce the $k$-GNB model which allows even more complex dependence structure between the continuous explanatory variables.  In Section~\ref{Sec:Classification} two classification procedure are given. The first one uses all explanatory feature. We give a reduction method based on $f1$ score which uses a subset of features and based on this we introduce the second classification procedure.
In section \ref{Sec:Numerical results} numerical results and comparisons with other glass-box algorithms are presented. We close the paper with conclusions and future work \ref{sec:Conclusions}.

\section{Preliminaries}
\label{sec:preliminaries}
Naive Bayes is a fast and effective classification method with a clear theoretical foundation, although its primary drawback is the conditional independence assumption. 

In this paper, we will use the extension of the Naive Bayes structure to the Generalized Naive Bayes (GNB) structure by allowing dependencies between the explanatory variables. However, we want to exploit beside their dependencies the conditional independencies between them also. The GNB structure was introduced as a special case of the discrete cherry trees in a recent paper \cite{kovacs2025generalized}. Now, we approach the case when the explanatory variables are continuous and are not conditionally independent given the class variable. Previous papers concerned with allowing dependence between explanatory variables typically do only discuss on conditional independence when they relate the structure to Bayes networks, which are described by directed acyclic graphs \cite{perez2006supervised}, \cite{perez2009bayesian}, \cite{friedman1998bayesian},\cite{wang2022semi}.

The GNB model with continuous explanatory variables is a special probabilistic mixed graphical model that includes both a discrete target variable and continuous explanatory variables. Therefore the concepts we use are coming from different fields like probability theory, graph theory, information theory, and optimization. 

However, in general the structure learning of the probabilistic graphical models is challenging, due to the combinatorial search required across all potential structures is NP complete \cite{chickering1996learning}.
In our approach the dimension of the marginal distributions involved remains bounded, this way, the search space for finding the best-fitting structure to the real distribution described by the sample data, is also massively reduced.

The following subsections summarize key concepts and preliminary results essential for the new findings presented in the paper.

\subsection{The GNB structures}
The GNB graph structure was introduced in \cite{kovacs2025generalized} as a special case of the cherry-tree structure introduced in \cite{kovacs2010approximation} \cite{szantai2012hypergraphs}, see  \ref{Appendix:cherry tree} for the general definition of the cherry tree. Since we use the same graph structure as in \cite{kovacs2025generalized} we define it directly.

\begin{definition}
\label{def:GB_gstructure}
A  \textbf{Generalized Naive Bayes graph-structure} on $V=D\cup\left\{0\right\}=\\$ $ =\left\{ 0,1,\ldots ,d\right\}$ (vertex 0 is assigned to $Y$), is constructed as follows:
\begin{itemize}[itemsep=-2mm,topsep=-2mm]
\item Step 1. The smallest GNB structure on three vertices is given by the interconnected triplet $\left(
0,i_{1},i_{2}\right) $, where $i_{1},i_{2},\in D$. 
\item Step k. Let us suppose that the vertices $i_{1},i_{2},\ldots
,i_{k}\in D$ are already in the GNB structure. We add a new vertex $i_{k+1}\in D$ to the existing GNB structure by connecting it to vertex $0$ and an already connected
vertex $i_{m}\in \left\{ i_{1},i_{2},\ldots ,i_{k}\right\} $. The vertex $i_m$ is
called the  \emph{mother of vertex} $i_{k+1}, k>0$, denoted by $\mu (i_{k+1})$.
\end{itemize}
\end{definition}

\begin{figure}
    \centering
    \includegraphics[width=0.8\textwidth]{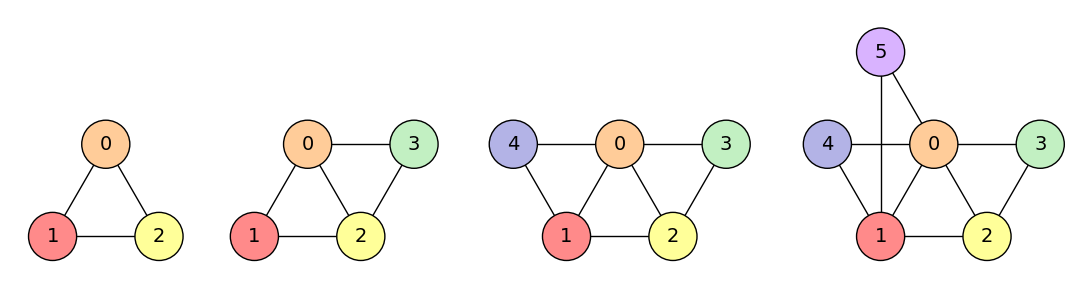}
    \vspace{-2mm}
    \caption{Construction of a GNB graph on 6 vertices. $0$ corresponds to the target variable while the other vertices to the explanatory variables}
    \label{fig:GNB_constr}
    \vspace{-3mm}
\end{figure}

We call \textbf{GNB-junction tree } the structure assigned to the GNB graph structure the following way:

\begin{enumerate}[itemsep=-2mm, topsep=-2mm] 
\item  The set of vertices in each $3$- element maximum clique (fully connected subgraph) is called a \emph{cluster} containing its $3$ vertices (a node in the junction tree graph).
\item Two $3$-element clusters are connected if the following two conditions are fulfilled: (a) the clusters share $2$ elements; (b) if a set of $2$ elements is contained by $m$ clusters, these clusters will be connected tree-like by $m-1$ edges, see second row of Figure \ref{fig:GNB_junction}.
\item The $2$-element set given by the intersection of two connected clusters is called a \emph{separator}. The number of clusters containing it is called the \emph{multiplicity of the separator}.
\end{enumerate}

For an illustration of a GNB-structure, corresponding to a GNB graph, see first row in Figure \ref{fig:GNB_junction}. The junction trees associated with a GNB graph structure in second row of Figure \ref{fig:GNB_junction} do not have a unique graph representation, but all can be represented in a compact form Figure \ref{fig:GNB_junction}, third row.

\begin{figure}[h]
    \centering
\includegraphics[scale=0.5]{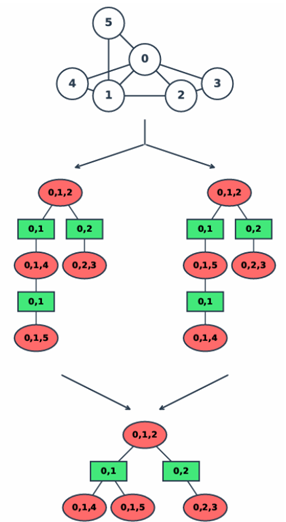}
    \vspace{-2mm}
    \caption{ The GNB structure represented as a chordal graph (first row), corresponding GNB junction trees (second row) and represented in a compact junction tree form (third row). The clusters are colored in red, the separators are colored in green.}
    \label{fig:GNB_junction}
\end{figure}

Let $\mathbf{X}=\left( X_{1},\ldots ,X_{d}\right)^{T}$ \ a random vector of the explanatory variables
and  $D$ $=\{1,...,d\}$ the set of indices. The GNB-junction tree over $V= D\cup 0$ given by the set of clusters $\mathcal{C}$, the set of separators $\mathcal{S}$ together with the separators multiplicity set $\mathcal{M}$ denoted by $(V,\mathcal{C},\mathcal{S},\mathcal{M})$.
For convenience we will use the following notations: the probability distribution of a random vector with its components $( X_{i_1},\ldots ,X_{i_k})$ is shortened as follows: $p(\mathbf{X}_{i_1,\ldots ,i_k})$ and $p(\mathbf{X}_A)$ where $A=\{i_1,\ldots ,i_k\}\subset V$.

The probability distribution corresponding to a GNB-junction tree $(V,\mathcal{C},\mathcal{S},\mathcal{M})$ is given by the following formula:
\begin{equation}
\label{eq:Cherry_tree}
p_{GNB}( \mathbf{x}) =\frac{\prod\limits_{C\in \mathcal{C}}p\left( \mathbf{x}_{C}\right) }{\prod\limits_{S\in \mathcal{S}}p\left( \mathbf{x}_{S}\right)
^{v_{S}-1}}=\frac{\prod\limits_{(l,m,0)\in \mathcal{C}}p\left( \mathbf{x}_{l,m,0}\right) }{\prod\limits_{(s,0)\in \mathcal{S}}p\left( \mathbf{x}_{s,0}\right)
^{v_{S}-1}}
\end{equation}
where $v_{S}$ is the number of clusters that are connected through the
separator $S=(s,0)$ (multiplicity of $S$); the formula stands for all realizations of $\mathbf{X}$.

%Without loss of generality, this defintion can be written also as:
%\begin{definition}
%\label{GNB_pd}
%The  \emph{GNB probability distribution} associated to the GNB graph structure and a $c$-ordering $1,\ldots,d$ is the following:
%\[
%P_{GNB}=\frac{\prod\limits_{i=2}^{d}P\left( Y,X_{\mu
%(i)},X_{i}\right) }{\prod\limits_{i=3}^{d-1}P\left( Y,X_{_{\mu
%(i)}}\right) }.
%\]
%\end{definition}

One can note that, regardless of which graphical representation of a junction tree we use (see, for example, Figure \ref{fig:GNB_junction}), the probability distribution associated has the same formula. In our example this is:
\[
P_{GNB}\left(\mathbf{x}\right) =\frac{p( \mathbf{x}_{0,1,2}) p( \mathbf{x}_{0,2,3}) p( \mathbf{x}_{0,1,4})p( \mathbf{x}_{0,1,5})}{(p\left( \mathbf{x}_{0,1}\right))
^{2}p\left( \mathbf{x}_{0,3}\right)}.
\]

\subsection{ Results related to copula theory}
In 1959 Abe Sklar introduced copula functions \cite{sklar1959fonctions}, which are capable of modeling the dependence between variables separate from the one-dimensional marginals.

Let $\boldsymbol{X} = (X_1,\dots,X_d) \in \mathbb{R}^d$ be an $d$-dimensional continuous random vector, with $p$ joint p.d.f. and $F$ joint c.d.f., meaning that
$$F(x_1,\dots,x_d) = \mathbb{P}(X_1 \le x_1, \dots, X_d \le x_d).$$

\begin{definition} A function $C: [0,1]^d \to [0,1]$ is an $d$-dimensional copula function, if it satisfies the following properties:
\begin{itemize}
\item $C(u_1,\dots,u_d)$ is strictly increasing in all $u_i$ components.
\item $C(u_1,\dots,u_{i-1},0,u_{i+1},\dots,u_d) = 0$ for all $u_k \in [0,1]$, $k \ne i$, $i \in \{1,\dots,d\}$.
\item $C(1,\dots,1,u_i,1,\dots,1) = u_i$ for all $u_i \in [0,1]$, $i \in \{1,\dots,d\}$.
\item $C$ is $d$-increasing, meaning that for all $(u_{1,1},\dots,u_{1,d})$ and $(u_{2,1},\dots,u_{2,d})$ in $[0,1]^n$, if for all $i$, $u_{1,i} < u_{2,i}$, then

$$\sum_{i_1 = 1}^2 \cdots \sum_{i_d = 1}^2 (-1)^{\sum_{j=1}^d i_j} C(u_{i_1,1},\dots,u_{i_d,d}) \ge 0$$
\end{itemize}

\end{definition}

In the theory of copulas, Sklar's theorem (1959) \cite{sklar1959fonctions} can be regarded as a central theorem:

\begin{theorem}
\label{sklar_theorem} Any multivariate c.d.f. $F$ can be written in the following way:

$$F(x_1,\dots,x_n) = C(F_1(x_1),\dots,F_n(x_n))$$

Moreover, if the one-dimensional marginals $F_1,\dots,F_d$ are continuous, then $C$ in unique.
\end{theorem}
 The following formula connects the original density function with the copula density: 
\begin{equation}
\label{eq:joint_f}
p(x_1,\dots,x_d) = c(F_1(x_1),\dots,F_n(x_d)) \cdot p_1(x_1) \cdot \dots \cdot p_d(x_d)
\end{equation}
Where $c$ is the joint p.d.f. of $\boldsymbol{U} = (U_1,\dots,U_n)$ or the mixed partial derivative of $C$. 

An interesting overview is given in \cite{sklar1996random}
\subsection{Information theoretical concepts}
\label{subsec:inf_theor_concepts}
Since the best-fitting GNB probability distribution is achieved by minimizing the Kullback-Leibler divergence we recall the necessary information theoretical concepts. For more details, see the fundamental book \cite{cover2012elements}. 

We will use an equivalent formulation of \textbf{conditional entropy} that quantifies the amount of information needed to describe the outcome of a random variable $Y$  given a random variable $X$: $H(Y|X)=H(X,Y)-H(X)$.

We will use the following formula for the information content of a bi-variate mixed random variable consisting of $X$ continuous and $Y$ discrete random variables:
\begin{equation}
\label{information- entropy}
    I(\mathbf{X}) = H(X) + H(Y) - H(X, Y)
\end{equation}
where: $H(Y)$ stands for the entropy of a discrete random variable $Y$: $H(Y) = -\sum_{y \in \mathcal{Y}} p(y) \log p(y)$; $H(X, Y)$ is the joint entropy of a continuous variable $X$ and a discrete variable $Y$: $H(X, Y) = -\sum_{y \in \mathcal{Y}} \int_{\mathcal{X}} p(x, y) \log p(x, y) \,dx$, and $\chi$ is the domain of the random variable $X$

When a $p(\mathbf{x})$ probability density is approximated by $p_{app}(\mathbf{x})$ the quantification of the information loss is given by the \textbf{Kullback-Leibler divergence} (KL-divergence) \cite{kullback1951information}:
\begin{equation}
 \mathbf{KL}(p(\mathbf{x}),p_{app}(\mathbf{x}))=\int_{\chi}p\left( \mathbf{x}\right) \log_2 \frac{p\left( \mathbf{x}\right) }{p_{app}(\mathbf{x})}d\mathbf{x}.   
\end{equation}

KL divergence is a positive number unless $P_{app}\equiv P$, when $\mathrm{KL}=0$. The smaller the value of the KL-divergence is the better the approximation is.

The \textbf{mutual information} can be seen as a special type of KL divergence between a given $p(X_{i_1},X_{i_2})$ and an approximation given by independence: $p_{app}=p(X_{i_1})p(X_{i_2})$.

The \textbf{information content} of a random vector $\mathbf{X}=(X_{i_{1}},...,X_{i_{k}})$ can be given straightforward as a generalizations of the mutual information:
\begin{equation}
    \label{Inf_Content_density}
    I(X_{i_{1}},...,X_{i_{k}})=\int_{\mathbf{\chi}}p\left( x_{i_{1}},\dots,x_{i_{k}}\right) \log_2 \frac{p\left( x_{i_{1}}\mathbf,\dots,x_{i_{k}}\right) }{p(x_{i_{1}})\cdot \ldots \cdot p(x_{i_{k}})}d\mathbf{x}.
\end{equation}
or more general by the following formula
\begin{equation}
\label{Inf_Content}
    I(\mathbf{X})=\sum\limits_{t=i_1}^{i_k}H(X_{t})-H(\mathbf{X})
\end{equation}
We note that the information content defined above and used in the papers \cite{kovacs2010approximation} and \cite{szantai2012hypergraphs} quantifies the strength of the dependence between the random variables; it was also called by other authors as multiinformation or total correlation.

\vspace{2mm}

A fundamental theorem we use in this paper is an adaptation of the general result of the KL divergence between the real probability distribution and an approximation given by a general cherry tree \cite{szantai2012hypergraphs}. We are now interested in the case when the approximation is given by a GNB junction tree. The expression of the KL divergence between a real probability distribution and a GNB junction tree probability distribution involves marginals of order two and three corresponding to the separators and clusters, see \cite{kovacs2025generalized}: 

\begin{theorem}
The Kullback-Leibler divergence between the real distribution $P(\mathbf{X})$ and the GNB junction tree approximation (\ref{eq:Cherry_tree}) is:
\begin{align}
\label{eq:KL_general}
KL(P\left(  \mathbf{X}Y\right)  ,P_{GNB}\left(  \mathbf{X}Y\right)  ) &
=\sum\limits_{i=1}^{d}H(X_{i})+H(Y)-H(\mathbf{X,}Y)\notag \\
&  -\left(  \sum\limits_{C\in\mathcal{C}}I\left(\mathbf{X_C}\right)
-\sum_{S\in\mathcal{S}}(\nu_s-1)I\left(  \mathbf{X_S}\right)
\right),
\end{align}
where the clusters $C$ and the separators $S$ contain the index 0 corresponding to the target variable.
\end{theorem}
Since KL depends on the structure of the GNB only through the amount in the parenthesis, we introduce the following definition.
\begin{definition}
    We call the weight of the approximation the following sum:
    \begin{equation}
    \label{weight of cherry}
        W=\sum\limits_{C\in \mathcal{C}}I(\mathbf{X}_C)-\sum%
\limits_{S\in \mathcal{S}}\left( \nu _{s}-1\right) I(\mathbf{X}_S)
    \end{equation}
\end{definition}
Since the weight is positive or zero, a larger weight results in a better approximation. Therefore, we conclude that it is worthwhile to find the structure with the largest weight.

In \cite{kovacs2025generalized} it was proved that the GNB approximation provides an approximation to the training data that is at least as good as that of the NB approximation

In the present paper we deal with continuous features, which do not alter the formula and our objective because all information contents are positive or $0$. (Differential entropy for example does not preserve this property in the continuous case.)
We aim to maximize the weight; a problem which occurs is that information contents depend on a discrete random variable (the classifier) and continuous random variables at the same time. Therefore we have to introduce the information content of a mixed random vector.

\vspace{2mm}

\textbf{The information contents of a mixed random vector}

We introduce the following notations which are used throughout the paper:
$p(\mathbf{x}, y)$ is the joint probability density/mass function of $\mathbf{X}$ and $Y$, $p_i(x_i)$ denotes the uni-variate marginal probability density function of $X_i$, whereas $p(\mathbf{x})$ denotes the multivariate probability density function of the random vector $\mathbf{X}=(X_1,...,X_d)$, $p(y)$ is the probability distribution of the discrete random variable (classifier)  \(Y\).
%In this section we derive the formulas, we need for our next sections.
In our present setting $\mathbf{X}$ the random vector containing continuous attributes and the discrete, target random variable $Y$ that takes on values in the set $\mathcal{Y}=\{y_1,...,y_m\}$.

\begin{theorem}
The formula for the information content of a random vector consisting of $d$ continuous explanatory variables $\mathbf{X}=(X_1,..., X_d)$ and a discrete variable $Y$ is the following:
\begin{equation}
\label{information content general}
I(\mathbf{X},Y) =\sum\limits_{i=1}^{d}H(X_{i})-\sum\limits_{y_{k}\in \mathcal{Y}}p(y_{k})H(\mathbf{X}|y_{k}).
\end{equation}
\end{theorem}
For the proof see \ref{Appendix:Information content}

\section{The weight of a GNB approximation under different assumptions on the joint distribution of the explanatory variables}
\label{sec:gnb_weight}
In the Subsection \ref{subsec:inf_theor_concepts}, we pointed out that the weight of the approximation of the approximation connects the GNB structure to the goodness of fit. In this chapter we examine three cases. In the first case, we assume that the explanatory variables follow a multivariate Gaussian distribution then we discuss the general case when the joint probability distribution is expressed by the means of copulas, with arbitrary marginal distributions.
In the last case we consider Gaussian copula to model the dependence structure, and arbitrary marginals to model the features.

\subsection{Features following joint Gaussian distribution}
\label{Subsec: joint Gauss assumption}
In this subsection, we suppose the target to be a discrete variable and the explanatory variables to have a joint Gaussian distribution with mean vector $\mu_\mathbf{X}$ and covariance matrix $\mathbf{K}_\mathbf{X}$. This seems to be quite restrictive, but there are many cases when this works very well. 
In order to calculate the weight of the approximation, the mixed information content of orders two and three have to be calculated. We also will use that all marginals of a Gaussian distribution are Gaussian.

\vspace{2mm}
\textbf{The weight of the Gauss-GNB structure}

%ide jön az egyváltozós entrópia
    
First, we derive the formula for the information content of a bi-variate random vector consisting of a continuous Gaussian random variable $X$ and a discrete random variable $Y$. For convenience, we use the natural logarithm. The uni-variate Gaussian entropy is given by formula: $H(X) =\frac{1}{2}\log (2\pi e\sigma ^{2})$.

The mutual information between a continuous variable $X$ with Gauss distribution and a discrete variable $Y$ is given by the following formula. (The proof is given in \ref{Appendix: I(X,Y)_Gauss_proof})
\begin{eqnarray}
\label{I(X,Y)_Gauss}
I(X,Y) =\frac{1}{2}\left(\log \sigma^2_{X} -\sum\limits_{y_{k}\in \mathcal{Y}%
}p(y_{k})\log \sigma ^2
_{X|y_{k}}\right)
\end{eqnarray}

Now, we deduce a general formula for the information content of a continuous multivariate Gaussian distribution over $\mathbf{X}$ with covariance matrix $\mathbf{K}$ and a discrete random variable $Y$. 
We start from formula \ref{information content general} and use besides the formula of the entropy of a uni-variate Gaussian distribution the entropy of the $d$ dimensional Gaussian distribution:
\begin{eqnarray*}
H(\mathbf{X}) &=&\frac{1}{2}\left[ \log (2\pi e)^{d}\left\vert K_\mathbf{X}
\right\vert \right]
\end{eqnarray*}
\begin{eqnarray}
\label{I(bfX,Y)_Gauss}
I(\mathbf{X},Y) &=&\sum\limits_{i=1}^{d}H(X_{i})-\sum\limits_{y_{k}\in 
\mathcal{Y}}p(y_{k})H(\mathbf{X}|y_{k})= \notag \\
&=&\sum\limits_{i=1}^{d}\frac{1}{2}\log (2\pi e\sigma
_{X_i}^{2})-\sum\limits_{y_{k}\in \mathcal{Y}}p(y_{k})\left( \frac{1}{2}\log
\left( 2\pi \right) ^{d}+\frac{1}{2}\log \left\vert K_{\mathbf{X}|y_{k}}\right\vert +%
\frac{d}{2}\right) = \notag \\
&=&\frac{1}{2}\left[ \sum\limits_{i=1}^{d}\left( \log 2\pi +1+\log \sigma_{X_i}^{2}\right) -\left( d\log \left( 2\pi \right) +d\right)\sum\limits_{y_{k}\in \mathcal{Y}}p(y_{k})\right] - \notag \\
&&-\frac{1}{2}\sum\limits_{y_{k}\in \mathcal{Y}%
}p(y_{k})\log \left\vert K_{\mathbf{X}|y_{k}}\right\vert=  \notag\\
&=& \frac{1}{2}\left[ \sum\limits_{i=1}^{d}\log \sigma
_{X_i}^{2}-\sum\limits_{y_{k}\in \mathcal{Y}}p(y_{k})\log \left\vert
K_{\mathbf{X}|y_{k}}\right\vert \right] 
\end{eqnarray}

As we saw in the preliminaries, KL divergence depends on the GNB structure only through its weight, which depends on the information content of the clusters and the separators in the GNB structure \ref{weight of cherry}.

As a particular case for formula \ref{I(bfX,Y)_Gauss} we express  $I(X_{l},X_{m},Y)$ as 
\begin{eqnarray}
\label{I(X,Y,Z)_Gauss}
I(X_{l},X_{m},Y) =\frac{1}{2}\left[ \log \sigma _{l}^{2}+\log \sigma
_{m}^{2}-\sum\limits_{y_{i}\in \mathcal{Y}}p(y_{i})\log
|K_{X_{l}X_{m}|y_{i}}|\right] 
\end{eqnarray}

Based on \ref{I(X,Y)_Gauss} and \ref{I(X,Y,Z)_Gauss} we will deduce the formula of weight.
    
\begin{theorem}
\label{W_GNB_Gauss}
The weight of a GNB distribution corresponding to the $(V,\mathcal{C},\mathcal{S},\mathcal{M})$ GNB structure with joint Gaussian explanatory variables is:
\begin{eqnarray}
\label{weight_Gauss}
W&=&\frac{1}{2}\log \prod\limits_{m\in D}\sigma _{m}^{2}-\frac{1}{2}\underset%
{y_{i}\in \mathcal{Y}}{\sum }p(y_{i})\log \left( \prod\limits_{m\in
D}\sigma _{m|yi}^{2}\right) + \notag \\
&-&\frac{1}{2}\underset{y_{i}\in \mathcal{Y}}{%
\sum }p(y_{i})\sum\limits_{(l,m,0)\in \mathcal{C}}\log (1-\rho^2 _{lm|y_{i}})
\end{eqnarray}
\end{theorem}
where $\rho^2 _{lm|y_{i}}$ stands for the square of the Pearson correlation between $X_l$ and $X_m$ filtered for the elements where the target $Y$ takes on the value $y_i$.
\begin{proof}

We write the formula of weight \ref{weight of cherry} in a particular form:

\begin{eqnarray}
W &=&\sum\limits_{(l,m,0)\in \mathcal{C}}I(X_{l},X_{m},Y)-\sum%
\limits_{(l,0)\in \mathcal{S}}\left( v_{\left( l,0\right) }-1\right)
I(X_{l},Y)= \notag \\
&=&\frac{1}{2}\sum\limits_{(l,m,0)\in \mathcal{C}}\left[ \log \sigma
_{l}^{2}+\log \sigma _{m}^{2}-\underset{y_{i}\in \mathcal{Y}}{\sum }%
p(y_{i})\log |K_{X_{l}X_{m}|y_{i}}|\right] - \notag \\
&&-\frac{1}{2}\sum\limits_{(s,0)\in \mathcal{S}}\left( v_s
-1\right) \left[ \log \sigma _{s}^{2}-\underset{y_{i}\in \mathcal{Y}}{\sum 
}p(y_{i})\log \sigma^{2}_{s|y_{i}}\right) 
\end{eqnarray}

Due to the junction tree structure each vertex appears once more in the clusters than in the separators this is why each term $\sigma_{i}, i=1\dots d$ appears a single time in the next formula. 
\begin{eqnarray}
\label{weight}
W &=&\frac{1}{2}\sum\limits_{m\in D}\log \sigma _{m}^{2}-\frac{1}{2}%
\sum\limits_{(l,m,0)\in \mathcal{C}}\underset{y_{i}\in \mathcal{Y}}{\sum }%
p(y_{i})\log |K_{X_{l}X_{m}|y_{i}}|+ \notag \\
&&+\frac{1}{2}\sum\limits_{(s,0)\in \mathcal{S}}\left( v_s-1\right) \underset{y_{i}\in \mathcal{Y}}{\sum }p(y_{i})\log \sigma_{s}^{2}|y_{i} =\frac{1}{2}\sum\limits_{m\in D}\log \sigma _{m}^{2}-\notag \\
&&-\frac{1}{2}\left[ 
\underset{y_{i}\in \mathcal{Y}}{\sum }p(y_{i})\left(
\sum\limits_{(l,m,0)\in \mathcal{C}}\log
|K_{X_{l}X_{m}|y_{i}}|-\sum\limits_{(l,0)\in \mathcal{S}} (v_
s-1) \log \sigma^{2} _{s|y_{i}}\right) \right] 
\end{eqnarray}

We calculate the determinant of the covariance matrix for a given $y_{i}$:
\begin{equation*}
|K_{X_{l}X_{m}|y_i}|=\det \left( 
\begin{array}{cc}
\sigma _{l|y_{i}}^{2} & \rho _{lm|y_{i}}\sigma _{l|y_{i}}\sigma _{m|y_{i}}
\\ 
\rho _{lm|y_{i}}\sigma _{m|y_{i}}\sigma _{l|y_{i}} & \sigma _{m|y_{i}}^{2}%
\end{array}%
\right) =\sigma _{l|y_{i}}^{2}\sigma _{m|y_{i}}^{2}(1-\rho^2 _{lm|y_{i}})
\end{equation*}

Applying the logarithm we obtain:
\begin{eqnarray*}
\log |K_{X_{l}X_{m}|y_{i}}| &=&\log \left[ \sigma _{l|y_{i}}^{2}\sigma
_{m|y_{i}}^{2}(1-\rho^2 _{lm|y_{i}})\right]  \\
&=&\log \sigma _{l|y_{i}}^{2}+\log \sigma _{m|y_{i}}^{2}+\log (1-\rho^2_{lm|y_{i}})
\end{eqnarray*}

The inner parenthesis of formula \ref{weight} becomes:
\begin{eqnarray*}
&&\sum\limits_{(l,m,0)\in \mathcal{C}}\log
|K_{X_{l}X_{m}|y_{i}}|-\sum\limits_{(l,0)\in \mathcal{S}}\left( v_{\left(
l,0\right) }-1\right) \log \sigma ^{2}_{l|y_{i}} \\
&=&\sum\limits_{(l,m,0)\in \mathcal{C}}\left( \log \sigma
_{l|y_{i}}^{2}+\log \sigma _{m|y_{i}}^{2}+\log (1-\rho _{lm|y_{i}})\right)
-\sum\limits_{(s,0)\in \mathcal{S}}( v_s -1)
\log \sigma ^{2}_{l|y_{i}}
\end{eqnarray*}

Now we apply the same idea as before exploiting the fact that terms of the form $\sigma_{l|y_{i}}^{2}$ appear in the clusters exactly once more than in the separators.

It follows that the amount of the inner parenthesis of \ref{weight} equals to:
\begin{equation*}
\sum\limits_{m\in D}\log \sigma _{m|yi}^{2}+\sum\limits_{(l,m,0)\in 
\mathcal{C}}\log (1-\rho _{lm|y_{i}}).
\end{equation*}%
Returning to the formula of the weight \ref{weight} we obtain:
\begin{eqnarray*}
W &=&\frac{1}{2}\sum\limits_{m\in D}\log \sigma _{m}^{2}-\frac{1}{2}\left[ 
\underset{y_{i}\in \mathcal{Y}}{\sum }p(y_{i})\left( \sum\limits_{m\in
D}\log \sigma _{m|yi}^{2}+\sum\limits_{(l,m,0)\in \mathcal{C}}\log (1-\rho^2
_{lm|y_{i}})\right) \right] = \\
&=&\frac{1}{2}\log \prod\limits_{m\in D}\sigma _{m}^{2}-\frac{1}{2}\underset%
{y_{i}\in \mathcal{Y}}{\sum }p(y_{i})\log \left( \prod\limits_{m\in
D}\sigma _{m|yi}^{2}\right) + \\
&-&\frac{1}{2}\underset{y_{i}\in \mathcal{Y}}{%
\sum }p(y_{i})\sum\limits_{(l,m,0)\in \mathcal{C}}\log (1-\rho^2 _{lm|y_{i}})
\end{eqnarray*}
\end{proof}
\begin{remark}
    One can see that the weight is connected to the GNB graph structure only through:
\begin{equation}
    \label{rho_Gauss}
\sum\limits_{(l,m,0)\in \mathcal{C}}\underset{y_{i}\in \mathcal{Y}}{\sum }%
p(y_{i})\log (1-\rho _{lm|y_{i}}) 
\end{equation}
which has to be minimized in order to maximize \ref{weight_Gauss}.  Therefore, we have to find the structure which minimizes \ref{rho_Gauss}
\end{remark}
Interesting to note here that this formula was also obtained in the paper \cite{perez2009bayesian}, by minimizing the description length.

\subsection{GNB with a general nonparametric joint continuous distribution between the explanatory variables} 
\label{sec:gnb_general}
Recall that the KL divergence between the joint distribution
$P\left(  \mathbf{X}Y\right)  $ and the approximation $P_{GNB}\left(\mathbf{X}Y\right) $ given by the formula \ref{eq:KL_general}, we rewrite it by using the notation $\mu(i)$ to denote the mother vertex of vertex $i$ (it follows from the construction):
\begin{align*}
KL(P\left(  \mathbf{X}Y\right)  ,P_{GNB}\left(  \mathbf{X}Y\right)  )  &
=\sum\limits_{i=1}^{d}H(X_{i})+H(Y)-\\
&  -\left(  \sum\limits_{(i,\mu\left(  i\right)  ,0)\in\mathcal{C}}I\left(
X_{i},X_{\mu\left(  i\right)  },Y\right)  -\sum_{(\mu\left(  i\right)
,0)\in\mathcal{S}}I\left(  X_{\mu\left(  i\right)  },Y\right)  \right) .
\end{align*}

One can observe, that the value of $KL(P\left(  \mathbf{X}%
Y\right)  ,P_{GNB}\left(  \mathbf{X}Y\right)  )$ depends on the GNB junction tree structure only by the amount contained in the parentheses called the weight of the GNB-structure:
\begin{equation}
\label{*}
W_{GNB}(\mathbf{X,}Y\mathbf{)=}\sum\limits_{(i,\mu\left(  i\right)
,0)\in\mathcal{C}}I\left(  X_{i},X_{\mu\left(  i\right)  },Y\right)
-\sum_{(\mu\left(  i\right)  ,0)\in\mathcal{S}}I\left(  X_{\mu\left(
i\right)  },Y\right) .
\end{equation}
There is no need to use the multiplicity of the separator since terms of the form $X_{\mu_(i)}$ may denote multiple times to the same variables.

We discuss the most general case, making no assumption on the copula or the marginals.
We prove the following essential lemma.

\begin{lemma}
\label{lemma}
The information content of a random vector  $\mathbf{X}$ with the probability density  $p(\mathbf{x})=c_\mathbf{X} ({F_1(x_1),\dots,F_d(x_d)})$ is equal with the information content of the vector $ \mathbf{U}$ described $c_\mathbf{X}(\mathbf{u})$
\begin{equation}
\label{copula information}
 I\left( \mathbf{X}\right) =I_{c_{\mathbf{x}}} (\mathbf{U})
\end{equation}
\end{lemma}

\begin{proof}
An important result published by Ma and Sun in their 2008/2011
paper \cite{ma2011mutual} is that the "Mutual Information is equal to copula entropy".

To unify the notations we now revise it as follows:
Let us denote by $c_{\mathbf{X}}(%
\mathbf{u})$ the copula density assigned to the random
vector $\mathbf{X}$, by defining vector  $\mathbf{U}^{\prime }$s coordinate in the following way: $U_{i}=F_{i}(x)$. 

Copula entropy is defined similar to differential entropy of a continuous random vector as follows:
\[
H_{c_{\mathbf{X}}}\left( \mathbf{U}\right) =\underset{(0,1)^{%
%TCIMACRO{\U{b4}}%
%BeginExpansion
{\acute{}}%
%EndExpansion
d}}{-\int }c_{\mathbf{X}}(\mathbf{u})\log c_{\mathbf{X}}(%
\mathbf{u})d\mathbf{u}.
\]

The formula of information content for a random vector with continuous component \ref{Inf_Content_density} is:
\begin{equation}
\label{I}
I\left( \mathbf{X}\right) =\underset{\chi }{\int }p(\mathbf{x})\log \frac{p(%
\mathbf{x})}{\underset{i=1}{\overset{d}{\prod }}p(x_{i})}d\mathbf{x}
\end{equation}

We use the following equation which follows directly from \ref{eq:joint_f}
\[
p(\mathbf{x})=c_{\mathbf{X}}(\mathbf{u})\underset{i=1}{\overset{%
d}{\prod }}p(x_{i})
\]%
where 
\[
\mathbf{u}=\left( F_{1}(x_{1}),\ldots ,F_{d}(x_{d})\right) 
\]
and substitute in \ref{I} :
\begin{eqnarray*}
I\left( \mathbf{X}\right)  &=&\underset{\chi }{\int }c_{\mathbf{X}}(\mathbf{u}) \underset{i=1}{\overset{d}{\prod }}%
p(x_{i})\log \frac{c_{\mathbf{X}}(\mathbf{u}) \underset{i=1}{\overset{d}{\prod }}p(x_{i})}{\underset{%
i=1}{\overset{d}{\prod }}p(x_{i})}d\mathbf{x} \\
&=&\underset{\chi }{\int }c_{\mathbf{X}}(\mathbf{u})\underset{i=1}{\overset{d}{\prod }}p(x_{i})\log c_{\mathbf{X}}(\mathbf{u}) d\mathbf{x.}
\end{eqnarray*}

Since $F_{i}(x_{i})=u_{i}$  this implies $p_{i}(x)dx_{i}=du_{i}$ it follows directly that: 

\begin{equation}
\label{copula entropy}
I(\mathbf{X})=-H_{c_{\mathbf{X}}}\left( \mathbf{U}\right) ,
\end{equation}
The result in formula \ref{copula entropy} was published in \cite{ma2011mutual}.

We now make a further essential step. We express the information content of the random
vector $\mathbf{U}$ corresponding to $\mathbf{X}$.

\[
I{c_{\mathbf{X}}}\left( \mathbf{U}\right)=\sum_{i=1}^{d}H(U_{i})-H_{c_{\mathbf{X}}}(\mathbf{%
U})
\]

Since $U_{i}$ are uniform distributions on $(0,1)$ their entropy are equal to 0, so:%
\begin{equation}
\label{copula info}
I_{c_{\mathbf{X}}}\left( \mathbf{U}\right)=-H_{c_{\mathbf{X}}}(\mathbf{U})
\end{equation}

Formula \ref{copula entropy} and \ref{copula info} implies the result of the lemma.
\end{proof}

\begin{remark}
    Information content of a continuous random vector $\mathbf{X}$ depends only on the copula which describes the joint dependence, not on the marginals. Therefore it can be estimated as the KL divergence between the transformed data $\mathbf{U}$ and the independence between the components, which in this case means the product of uniform marginals.
\end{remark}

In the next theorem we give the formula of the weight of the GNB approximation without any restriction on copula or marginals.

\begin{theorem}
\label{weight_general}
The weight of a fitted GNB probability distribution  is given by the formula  
\begin{align}
\label{9}
W_{GNB}&=\sum\limits_{i=1}^{d}H(X_{i})-\sum\limits_{y_{k}\in\mathcal{Y}%
}p(y_{k})\sum\limits_{i=1}^{d}H(X_{i}|y_{k})+\notag \\+&\sum\limits_{(i,\mu\left(
i\right)  ,0)\in\mathcal{C}}\sum\limits_{y_{k}\in\mathcal{Y}}p(y_{k}%
)I_{c_{X_{i},X_{\mu\left(  i\right)  }|y_{k})}}(U_i,U_{\mu(i)}).
\end{align}
\end{theorem}

\begin{proof}
To prove this, we use the expressions of $I\left(  X_{i}%
,X_{\mu\left(  i\right)  },Y\right)  $ and $I\left(  X_{\mu\left(  i\right)
},Y\right)  .$

We apply \ref{information content general} for the particular case:
\begin{equation}
\label{cluster info}
I\left(  X_{i},X_{\mu\left(  i\right)  },Y\right)  =H(X_{i})+H\left(
X_{\mu\left(  i\right)  }\right) -\sum\limits_{y_{k}\in\mathcal{Y}%
}p(y_{k})H(X_{i},X_{\mu\left(  i\right)  }|y_{k}).
\end{equation}

and%
\begin{equation}
\label{5}
    I\left(  X_{\mu\left(  i\right)  },Y\right)  =H(X_{\mu\left(  i\right)
})-\sum\limits_{y_{k}\in\mathcal{Y}}p(y_{k})H(X_{\mu\left(  i\right)
}|y_{k})
\end{equation}

 As consequence of \ref{information- entropy} for a given $y_{k}\in\mathcal{Y}$ we have the the following expression of $H(X_{i},X_{\mu\left(
i\right)  }|y_{k}):$%
\begin{equation}
\label{cond-entropy}
H(X_{i},X_{\mu\left(  i\right)  }|y_{k})=H(X_{i}|y_{k})+H(X_{\mu\left(
i\right)  }|y_{k})-I(X_{i},X_{\mu\left(  i\right)  }|y_{k}).
\end{equation}

By Lemma \ref{lemma} for each $y_{k}\in\mathcal{Y}$ each information content of the form $I(X_{i},X_{\mu\left(
i\right)}|y_{k})$ can be expressed using the copula information formula \ref{copula information}:%
\begin{equation}
\label{2}
I(X_{i},X_{\mu\left(  i\right)  }|y_{k})=I_{c_{X_{i},X_{\mu\left(  i\right)  }|y_{k}}}(U_i,U_{\mu(i)})
\end{equation}

By substituting \ref{2} in \ref{cond-entropy} we get:%
\begin{equation}
\label{3}
H(X_{i},X_{\mu\left(  i\right)  }|y_{k})=H(X_{i}|y_{k})+H(X_{\mu\left(
i\right)  }|y_{k})-I_{c_{X_{i},X_{\mu\left(  i\right)  }|y_{k}}}(U_i,U_{\mu(i)})
\end{equation}

Now, we substitute \ref{3} in \ref{cluster info}:%
\begin{align}
\label{6}
I\left(  X_{i},X_{\mu\left(  i\right)  },Y\right)   &  =H(X_{i})+H\left(
X_{\mu\left(  i\right)  }\right) - \notag \\
&  \sum\limits_{y_{k}\in\mathcal{Y}}p(y_{k})\left[  H(X_{i}|y_{k}%
)+H(X_{\mu\left(  i\right)  }|y_{k})-I_{c_{X_{i},X_{\mu\left(  i\right)  }|y_{k}}}(U_i,U_{\mu(i)})\right] .
\end{align}

By substituting \ref{6} and \ref{5} in the formula of the weight \ref{*}:%
\begin{align}
\label{7}
W_{GNB} &  =\sum\limits_{(i,\mu\left(  i\right)  ,0)\in\mathcal{C}}%
H(X_{i})+H\left(  X_{\mu\left(  i\right)  }\right)  - \notag \\
&  -\sum\limits_{y_{k}\in\mathcal{Y}}p(y_{k})\left[  H(X_{i}|y_{k}%
)+H(X_{\mu\left(  i\right)  }|y_{k})-I_{c_{X_{i},X_{\mu\left(  i\right)  }|y_{k}}}(U_i,U_{\mu(i)})\right] \notag \\
&  -\sum_{(\mu\left(  i\right)  ,0)\in\mathcal{S}}\left[  H(X_{\mu\left(
i\right)  })-\sum\limits_{y_{k}\in\mathcal{Y}}p(y_{k})H(X_{\mu(
i)|y_{k}})\right] .
\end{align}

Let us simplify Formula \ref{7}; We now have a GNB junction tree with clusters of the form
$\left(  X_{i},X_{\mu\left(  i\right)  },Y\right) $ and separators of the form $\left( X_{\mu\left(  i\right)  },Y\right) $  .
Since these indices are vertices are in a junction tree, the
number of clusters which contain a given index $i$ is exactly larger by $1$
than the number of separators which contain it. Therefore by taking the
difference we will obtain in the weight only one time each $H(X_{i})$. This way the weight $W_{GNB}$ reduces to the following new formula:%
\begin{align}
\label{8}
W_{GNB} &  =\sum\limits_{i=1}^{d}H(X_{i}))-\notag \\
&  -\left\{  \sum\limits_{(i,\mu\left(  i\right)  ,0)\in\mathcal{C}}%
\sum\limits_{y_{k}\in\mathcal{Y}}p(y_{k})\left[  H(X_{i}|y_{k})+H(X_{\mu
\left(  i\right)  }|y_{k})-I_{c_{X_{i},X_{\mu\left(  i\right)  }|y_{k}}}(U_i,U_{\mu(i)})\right]  \right.  - \notag \\
&  -\left.  \sum_{(\mu\left(  i\right)  ,0)\in\mathcal{S}}\left[
\sum\limits_{y_{k}\in\mathcal{Y}}p(y_{k})H(X_{\mu\left(  i\right)  }%
|y_{k})\right]  \right\}  .
\end{align}

We transform the formula \ref{8} by grouping the terms containing $H(X_{i}|y_{k})$
and $H(X_{\mu\left(  i\right)  }|y_{k})$ together:%
\begin{align*}
W_{GNB} &  =\sum\limits_{i=1}^{d}H(X_{i})-\\
&  \left\{  \sum\limits_{y_{k}\in\mathcal{Y}}p(y_{k})\left(  \sum
\limits_{(i,\mu\left(  i\right)  ,0)\in\mathcal{C}}\left[  H(X_{i}%
|y_{k})+H(X_{\mu\left(  i\right)  }|y_{k})\right]  -\sum_{(\mu\left(
i\right)  ,0)\in\mathcal{S}}H(X_{\mu\left(  i\right)  }|y_{k})\right)
\right.  -\\
&  -\left.  \sum\limits_{(i,\mu\left(  i\right)  ,0)\in\mathcal{C}}%
\sum\limits_{y_{k}\in\mathcal{Y}}p(y_{k})I_{c_{X_{i},X_{\mu\left(  i\right)  }|y_{k}}}(U_i,U_{\mu(i)})\right\}  .
\end{align*}
Let us simplify the sum:%
\[
\sum\limits_{y_{k}\in\mathcal{Y}}p(y_{k})\left(  \sum\limits_{(i,\mu\left(
i\right)  ,0)\in\mathcal{C}}\left[  H(X_{i}|y_{k})+H(X_{\mu\left(  i\right)
}|y_{k})\right]  -\sum_{(\mu\left(  i\right)  ,0)\in\mathcal{S}}%
H(X_{\mu\left(  i\right)  }|y_{k})\right)  .
\]
For each realization of $Y=y_{k}$, we have
\[
\sum\limits_{(i,\mu\left(  i\right)  ,0)\in\mathcal{C}}\left[  H(X_{i}%
|y_{k})+H(X_{\mu\left(  i\right)  }|y_{k})\right]  -\sum_{(\mu\left(
i\right)  ,0)\in\mathcal{S}}H(X_{\mu\left(  i\right)  }|y_{k}).
\]
We now use a similar logic as before:  All the vertices are present in the
clusters and in the separators. Since this is a junction tree each of the vertices will be contained by one more cluster than
separators. Summarizing these observations Formula we have:%
\[
\sum\limits_{(i,\mu\left(  i\right)  ,0)\in\mathcal{C}}\left[  H(X_{i}%
|y_{k})+H(X_{\mu\left(  i\right)  }|y_{k})\right]  -\sum_{(\mu\left(
i\right)  ,0)\in\mathcal{S}}H(X_{\mu\left(  i\right)  }|y_{k})=\sum
\limits_{i=1}^{d}H(X_{i}|y_{k})
\]
and formula of the weight \ref{8} can be expressed as:
\begin{align*}
W_{GNB}&=\sum\limits_{i=1}^{d}H(X_{i})-\sum\limits_{y_{k}\in\mathcal{Y}%
}[p(y_{k})\sum\limits_{i=1}^{d}H(X_{i}|y_{k})]+ \\
&+\sum\limits_{(i,\mu\left(
i\right)  ,0)\in\mathcal{C}}\sum\limits_{y_{k}\in\mathcal{Y}}p(y_{k}%
)I_{c(X_{i},X_{\mu\left(  i\right)  }|y_{k})}(U_i,U_{\mu (i)})
\end{align*}

\end{proof}
\begin{remark}
    As a consequence of the Theorem \ref{weight_general} one can observe that the weight is connected to the GNB structure only through the term:
\begin{equation}
 \sum\limits_{(i,\mu\left(
i\right)  ,0)\in\mathcal{C}}\sum\limits_{y_{k}\in\mathcal{Y}}p(y_{k}%
)I_{c_{X_{i},X_{\mu\left(  i\right) }|y_{k}}}(U_i,U_{\mu(i)})
\end{equation}
\end{remark}
We emphasize here that we have to estimate information contents  of pair copulas of vectors $(U_i,U_j)\in (0,1)^2$, which do not depend on marginals.

\begin{remark}
    Theorem \ref{weight_general} shows that GNB structure learning in the continuous case depends only on the information of the bi-variate copulas associated to the bi-variate marginals of the features. This shifts structure learning to an entirely new perspective.
\end{remark}

\subsection{Joint dependence structure characterized by Gauss copula}
\label{subsec:GNB gauss copula}
In subsection \ref{Subsec: joint Gauss assumption} we discussed the case when features are supposed to have a joint Gaussian distribution. This subsection explores a more flexible scenario when the explanatory variables are continuous and may belong to different families and the dependence structure is described by a joint Gaussian copula. For the concept of Gaussian copula, see for example \cite{nelsen2006introduction}.

By $p_i(x_i)$ respectively $F_i(x_i)$ are denoted the marginal probability density function (p.d.f) respectively the marginal cumulative distribution function (c.d.f.) of the variable $X_i$. If we have only sample data we have to approximate them. Here we use kernel density estimations for approximate them. Of course it is possible to fit the data uni-variate marginals using parametric families also.

The aim is to find the best fitting GNB structure by fitting a GNB-Gauss copula to the transformed data. This requires the following transformations.
First, we fit to all uni-variate marginals a kernel p.d.f respectively a kernel c.d.f.. We denote these by $\tilde{p_i}$ respectively $\tilde{F_i}$
We transform the sample data as follows $u^j_i=\tilde{F}_i(x^j_i)$, where $x^j_i$ is the value taken by the $i$-th explanatory variable $X_i$ in the $j$-th sample. 
As a result we obtain a transformed data where the new random vector is $\mathbf{U}=(U_1,\dots,U_d)$, each of the component $U_i$ is a uniform distribution on range $(0,1)$. Now, by applying a new transformation, namely by the inverse of the standard normal c.d.f.:
$z^j_i=\phi^{-1}(u^j_i) $ we obtain the random vector $\mathbf{Z}=(Z_1,\dots,Z_d)$. 

The next step is to fit to the transformed data corresponding to the random vector $\mathbf{Z}$ a Gaussian GNB distribution. We apply the methodology used in the previous subsection to discover the optimal Gaussian GNB structure in the sense of maximizing its weight over all Gaussian GNB structures. 
\begin{theorem}
\label{W_GNB_Gauss_copula}
The weight of a GNB approximation, when the dependence between features are described by a Gaussian copula is given by:

\begin{eqnarray*}
W_{GNB}^{_{^{c_{G}}}}
&=&\sum\limits_{i=1}^{d}H(X_{i})-\sum\limits_{y_{k}\in \mathcal{Y}%
}p(y_{k})\sum\limits_{i=1}^{d}H(X_{i}|y_{k})+ \\
&&-\sum\limits_{(i,j,0)\in \mathcal{C}}\sum\limits_{y_{k}\in \mathcal{Y}%
}p(y_{k})\log (1-(\rho _{z_i,z_j|k\text{ }}^{c_{G}})^2)
\end{eqnarray*}
where $\rho _{i,j|k\text{ }}^{c_{G}}$ stands for correlation coefficient of the bi-variate Gauss copula corresponding to the transformed data $z_i$ and $z_j$ after filtering for a given $y_k$ class.
\end{theorem}
\begin{proof}
We will uses formula \ref{9}, for which we need to calculate $%
I_{c_{X_{i}X_{j}|y_{k}}}(U_{i|k},U_{j|k})$ in the case when we suppose to describe the data by Gauss copula given each class.

For each $y_{k}$ we fit to the data $x_{i}^{s}|y_{k},i=1,\ldots d,s=1,\ldots
,n$ (lower index stands for the feature, upper index, for the sample vector) the kernel c.d.f. estimation $\widehat{F}%
_{i}(x_{i}|y_{k})$. We denote $\widehat{F}_{i}(x_{i}^{s}|y_{k})=u_{i|k}^{s}%
\in \left( 0,1\right) $. Now we apply the inverse standard cumulative normal
distribution and obtain $\Phi ^{-1}\left( u_{i|k}^{s}\right) :=z_{i|k}^{s}$.
We highlight here that in this way for all $i=1,\ldots d$ the obtained $%
z_{i|k}$ have a standard normal distribution. With this transformations the
initial sample vector $\left( x_{1|k}^{s},\ldots ,x_{d|k}^{s}\right) $ is
transformed into $\left( z_{1|k}^{s},\ldots ,z_{d|k}^{s}\right) ,s=1,\ldots
,n$, will have standard normal marginals and the joint distribution is given
by their correlation matrix. Let us denote by $\rho _{i,j|k\text{ }}^{c_{G}}$%
the correlation between $Z_{i|k}$ and $Z_{j|k}$ under the Gaussian copula
assumption.

We know that the joint entropy of a bivariate Gaussian distribution is given
by the formula%
\begin{equation}
\label{entropy_zizj}
H(Z_{i},Z_{j})=\frac{1}{2}\log ((2\pi e)^{2}|K_{Z_{i}Z_{j}}|)=\log (2\pi
e)+\log \left[ 1-\left( \rho _{z_i,z_j|y_k\text{ }}^{c_{G}}\right) ^{2}\right] 
\end{equation}

The entropy of the standard normal marginals are 
\begin{equation}
\label{entropy_zi}
H(Z_{i})=\frac{1}{2}\log (2\pi e)
\end{equation}

By substituting for a fixed class $y_{k}$ the \ref{entropy_zizj} and \ref{entropy_zi} in the formula of information content \ref{information- entropy} we obtain:
\begin{equation}
\label{inf Gauss cop}
I_{c_{X_{i}X_{j}|y_{k}}}(Z_{i|k},Z_{j|k})=H(Z_{i})+H(Z_{i})-H(Z_{i},Z_{j})=-%
\log (1-(\rho _{z_i,z_j|y_k}^{c_{G}})^2).
\end{equation}
Having the formula of \ref{inf Gauss cop} and formula of the weight \ref{9}  it follows straight forward the formula given by the theorem.
\end{proof}

\begin{remark}
    The weight of a Gaussian GNB copula distribution depends on the graph structure $(V,\mathcal{C},\mathcal{S},\mathcal{M})$
only through the term 
\begin{equation*}
\sum\limits_{(i,j,0)\in \mathcal{C}}\underset{y_{k}\in \mathcal{Y}}{\sum }%
p(y_{k})\log (1-(\rho _{z_i,z_j|y_k}^{c_{G}})^2)
\end{equation*}
where the term $\rho _{z_i,z_j|y_k}^{c_{G}}$ stands for the correlation between the transformed variables $Z_l$ and $Z_m$ when the data is filtered such that the class variable takes on the value $y_k$.
\end{remark}

\section{Combinatorial background of the GNB and greedy structure learning algorithms}
\label{sec:combinatorial background}
Structure learning algorithm aims to find the best fitting GNB probability distribution to the training data. In the following we examine this problem from combinatorial point of view.
\subsection{Mathematical structure}
In general, finding the best fitting $k$ order cherry tree p.d. approximation (for the general definition see \ref{Appendix:cherry tree}) to a given probability distribution is NP complete, when $k>2$ \cite{chickering1996learning} problem. The GNB probability distribution is a special case of it since GNB is a cherry tree with $k=3$ that fulfills the restriction that all of its clusters contain a common vertex, which is the index of the classification variable besides two indices corresponding to two features (in this case continuous). We will show how this problem can be related to a matroid.

Let us recall that matroids are an abstraction of several combinatorial objects, as a generalization of vectors linear independence. The word matroid was introduced by Whitney in 1935 in his landmark paper "On the abstract properties of linear dependence" \cite{whitney1992abstract}.
\begin{definition}
\label{matroid}
    The pair $(E,\mathcal{I})$ is called matroid if  $\mathcal{I}\subseteq2^{E}$  the collection of subsets of $E$, called the independent sets satisfies following axioms: 
\begin{enumerate}
    \renewcommand{\labelenumi}{(\roman{enumi})}
\item $\mathcal{I}\neq \phi $
\item If $A\in \mathcal{I}$, and $B\subset A$ then $B\in \mathcal{I}$
\item If $A,B\in \mathcal{I}$ and $|A|<|B|\Rightarrow \exists e\in
B\backslash A$ such that $A\cup \{e\}\in \mathcal{I}$
\end{enumerate}
\end{definition} 

\begin{theorem}
\label{Matroid}
The set $E$ of edges of the form $(i,j), i,j\in D$  together, with the independence set defined by the subsets of $E$ present in a forest has a matroid structure. 
\end{theorem}

\begin{proof}
We consider set $D$ of vertices that correspond to the indices of the explanatory variables. We have to prove that the set of all edges in $E$ with indices in $D$, together with the  set $\mathcal{I}$ :
    \begin{equation}
         \mathcal{I}=\left\{ F\subseteq E\ |\ F\text{ defines a forest on }D\right\}
    \end{equation}
  has a matroid structure.

First, it is easy to see, the $\mathcal{I}$ contains the empty set, i.e. there is no edge between any of the vertices in $D$ (when the vertices are not connected). 
(In the probabilistic setting this happens when the feature are conditionally independent with respect to $Y$ wich is the case of NB). 

Also it is easy to verify that if $A\in \mathcal{I}$, and $B\subset A$ then $B\in \mathcal{I}$, because a subset of a forest is also a forest.

The last axiom, called augmentation axiom is not so trivial to prove.

We use the important identity that $\left\vert F\right\vert =\left\vert
D\right\vert -c(F)$, where $F$ denotes a forest on all vertices in D and $%
c(F)$ denotes the components of $F$, which are trees or isolated vertices.

If $\left\vert A\right\vert <|B|\Rightarrow \left\vert D\right\vert
-c(A)<\left\vert D\right\vert -c(B)$ which implies that $c(A)>c(B)$ that is
the number of components of $A$ is greater than the number of components of $%
B$.

Consider adding edges from $B\backslash A$ to $A$. Let us denote an edge $%
e=(i,j),e\in B\backslash A$. At this point two possibilities may occur:

\begin{itemize}
\item There exist an edge $e\in B\backslash A$ which connects two components
of $A$ (to trees in A, a tree with an isolated vertex, or two isolated
vertices). In this case it remains a forest, which completes the proof.

\item There is no  $e\in B\backslash A$ edge between two different components of A (case 1). If so,  then all edges of B are within the vertices contained by the components of A (in the trees of A). This implies that $c(B)\geqslant c(A)$ this is contradiction with $c(A)>c(B).$
\end{itemize}
This means that only the first case is valid which, together with the verification of the first two axioms completes the proof.
\end{proof}

In \cite{kovacs2025generalized} it was proved that any GNB graph structure has the property that the subgraph determined by the explanatory variables is a tree. (see for example in Figure \ref{fig:GNB_constr}. 
So each GNB structure is characterized uniquely by the spanning tree between its explanatory variables.

\begin{theorem}
    The spanning tree structure which defines a GNB structure is a basis of the matroid defined.
\end{theorem}
\begin{proof}
The matroid defined in \ref{Matroid} has as its basis (i.e.the maximal independent set) the spanning trees of the complete graph.
\end{proof}
Based on this we have the following guarantee.
\begin{remark}
\label{matroid spanning tree}
    Finding a maximum weighted independent set (here a spanning tree) in a weighted matroid can be solved by a greedy algorithm.
\end{remark}

In our case we want to find the best fitting GNB structure, therefore a crucial step is how we define the weights assigned to the edges of the complete graph defined on the vertex set $D$ in order to find the spanning tree with the property that the GNB assigned has the maximal weight (best fitting in sense of KL minimization). 

We will discuss this for different assumptions on the joint probability distribution of the explanatory variables in the next subsection \ref{subsec:Greedy algorithms}.

\subsection{Greedy algorithms for structure learning}
\label{subsec:Greedy algorithms}
In the Subsections \ref{Subsec: joint Gauss assumption},\ref{sec:gnb_general}, \ref{subsec:GNB gauss copula} three cases of GNB structures under different assumptions were discussed. As mentioned earlier due to the matroid structure we have the guarantee of finding the basis with the largest weight, although the assignment of weights to the edges of a complete graph defined on ${1,\dots,d}$ is a crucial step. Now we put together the results we proved. For choosing the weights we rely on the results of the Theorems \ref{W_GNB_Gauss}, \ref{weight_general}  and \ref{W_GNB_Gauss_copula}.
\begin{enumerate}
    \item The case when the explanatory variables are supposed to have a \textit{joint Gaussian distribution} we assign to each edge $(l,m)$ the weight: 
\begin{equation}
\label{w_Gauss}
  \omega(l,m) =\underset{y_{i}\in \mathcal{Y}}{%
\sum }p(y_{i})\log (1-\rho^2 _{lm|y_{i}}),
\end{equation}
  where $\rho_{lm|y_i}$ denotes the correlation between the explanatory features $X_l$ and $X_m$ having the class value $Y=y_i$.  
    \item The case when the dependence structure is described by a \textit{Gauss copula}, we assign to each edge $(l,m)$ the weight:
    \begin{equation}
    \label{w_Gauss-copula}
       w_G(l,m)=\underset{y_{i}\in \mathcal{Y}}{\sum }%
p(y_{i})\log (1-{\rho_G}^2 _{lm|y_{i}}),
    \end{equation}
    where $\rho_G{_{lm|y_i}}$ denotes the parameter of the pair-Gauss copula which describes the dependence between $X_l$ and $X_m$ having the class value $Y=y_i$.
    \item The case when we have no restriction on the marginal distributions or on the joint copula called \textit{general distribution}, we assign to each edge $(l,m)$ the weight:
    \begin{equation}
    \label{w_general}
w(l,m)=\sum\limits_{y_{i}\in\mathcal{Y}}p(y_{i})I_{c_\mathbf{X}(U_{l},U_{m}|y_{i})}.
    \end{equation}
\end{enumerate}
We have the guarantee \ref{matroid spanning tree} of finding the best fitting structure to the training data (under a the given assumption of the dependence between the features) by applying Prim or Kruskal algorithms  to the complete graph with edges weighted by \ref{w_Gauss}, \ref{w_Gauss-copula}, or \ref{w_general}.

%Algoritmus-----------------------------------

\begin{algorithm}[H]
\caption{GNB-cont structure learning}\label{alg:gnb-cont}
\begin{algorithmic}
\Procedure{GNB-cont learning}{$D$, Method}
\State \textbf{Input:} Training samples $\mathbf{D}$ with $(X_1, \dots, X_d)$  and target $Y$, the Method \{GNBGauss, GNBGauss-cop, GNB-KDE\}
\State \textbf{Procedure:} 
%A sequence of vertex pairs together with their weights.

\State For all pairs $(l,m),l,m\in D, l\neq m$
\State Calculate the weights given by  formula \ref{w_Gauss}, 
 or \ref{w_Gauss-copula}, or \ref{w_general} as a function 
 \State of the method \hspace{10em}    $\blacktriangleright$ see Theorems \ref{W_GNB_Gauss}, \ref{weight_general}  and \ref{W_GNB_Gauss_copula}.
\State Apply Kruskal's algorithm: Find the maximum spanning tree of the \State complete graph on $D$
\State \textbf{Return}: Set of ordered clusters of the form $(l,m,0)$ based on the selection 
\State  order of pairs (l,m) connected in each step to the forest under construction.
\EndProcedure
\end{algorithmic}
\end{algorithm}

\subsection{Discussion on higher order dependencies between the attributes}
\label{subsec:higher order dependence}
The idea of allowing more than two explanatory variables to be interdependent was earlier discussed in paper like \cite{sahami1996learning}, \cite{perez2006supervised}, \cite{wang2022semi} respectively in a more recent paper \cite{ghofrani2018new}. All papers regarding the continuous explanatory feature space discuss this problem from the Bayesian network perspective. Bayesian networks are described by directed acyclic graphs. In this setting, increasing the degree of dependence among explanatory variables is achieved by relaxing the structural constraint that each variable may have only one parent besides the class variable. Ghofrani's paper \cite{ghofrani2018new} discusses this problem from undirected probabilistic graphical model perspective. However their approach is restricted to discrete features only.

Now we introduce the theoretical background of GNB-kind structures which allow higher order dependencies between the explanatory variables. We introduce the $k$-Generalized Naive Bayes ($k$-GNB), 
%and then we give the formulas of the weight of the approximation, under different assumptions.

\begin{definition}
\label{def:k-GB_gstructure}
A  \textbf{$k$-Generalized Naive Bayes graph-structure} on $D\cup\left\{0\right\}=\\$ $ =\left\{ 0,1,\ldots ,d\right\}$ (vertex 0 is assigned to $Y$), is constructed as follows:
\begin{itemize}[itemsep=-2mm,topsep=-2mm]
\item Step 1. The smallest $k$-GNB structure on $k$ vertices is given by a clique $\left(0,i_{1},\dots,i_{k-1}\right) $, where $i_{1},\dots,i_{k-1},\in D$. 
\item Step n. Let us suppose that the vertices $i_{1},i_{2},\ldots
,i_{n-k}\in D$ are already in the $k$-GNB structure. We add a new vertex $i_{n+1}\in D$ to the existing $k$-GNB structure by connecting it to a $k-1$ element clique which contain the vertex $0$. 
\end{itemize}
\end{definition}
We assign to a $k$-GNB graph structure a $k$-th order GNB junction tree with the property that each cluster contains $k-1$-vertices besides the vertex $0$.

The probability distribution assigned to the $k$-GNB junction tree $(V, \mathcal{C}, \mathcal{S}, \mathcal{M})$ is given by the same Formula \ref{eq:Cherry_tree}.The marginal probability distributions involved are now of size $k$ respectively $k-1$.

As an illustration, we discuss now only from theoretical point of view the case of joint Gaussian distribution.
First let we determine the formula for the weight of an approximation given by the \textit{Gaussian $k$-GNB probability distribution}.
\begin{theorem}
    The weight of a $k$-order GNB structure is given  as follows:
\begin{eqnarray}
\label{k-GNB weight}
W &=&\frac{1}{2}\sum\limits_{i=1}^{d}\log \sigma _{X_{i}}^{2}-\frac{1}{2}\left[ \sum\limits_{C\in 
\mathcal{C}}\sum\limits_{y_{k}\in \mathcal{Y}}p\left( y_{k}\right) \log
\left\vert K_{X_{C}}\right\vert \right. \notag \\
&&\left. -\sum\limits_{S\in \mathcal{S}}\left( \nu _{S}-1\right)
\sum\limits_{y_{k}\in \mathcal{Y}}p\left( y_{k}\right) \log \left\vert
K_{X_{S|y_{k}}}\right\vert \right] 
\end{eqnarray}
\end{theorem}
    
\begin{proof}
We express $I(\mathbf{X_C},Y)$ and $I(\mathbf{X_S},Y)$ by using Formula \ref{I(bfX,Y)_Gauss}:
\begin{equation}
\label{IC_Gauss}
I(\mathbf{X_C},Y) =\frac{1}{2}\left[ \sum\limits_{X_i\in\mathcal{C}}\log \sigma
_{X_i}^{2}-\sum\limits_{y_{k}\in \mathcal{Y}}p(y_{k})\log \left\vert
K_{\mathbf{X_C}|y_{k}}\right\vert \right] 
\end{equation}
\begin{equation}
\label{IS_Gauss}
I(\mathbf{X_S},Y) =\frac{1}{2}\left[ \sum\limits_{X_i\in S}\log \sigma
_{X_i}^{2}-\sum\limits_{y_{k}\in \mathcal{Y}}p(y_{k})\log \left\vert
K_{\mathbf{X_S}|y_{k}}\right\vert \right] 
\end{equation}
Now we substitute these in \ref{weight} and we get the weight:
\begin{eqnarray*}
W &=&\frac{1}{2}\sum\limits_{C\in \mathcal{C}}\left[ \sum\limits_{i\in C}\log \sigma
_{X_{i}}^{2}-\sum\limits_{y_{k}\in \mathcal{Y}}p\left( y_{k}\right) \log
\left\vert K_{X_{C}|y_{k}}\right\vert \right] - \\
&&-\frac{1}{2}\sum\limits_{S\in \mathcal{S}}\left( \nu _{S}-1\right) \left[
\sum\limits_{i\in S}\log \sigma _{X_{i}}^{2}-\sum\limits_{y_{k}\in 
\mathcal{Y}}p\left( y_{k}\right) \log \left\vert K_{X_{S|y_{k}}}\right\vert %
\right]  \\
&=&\frac{1}{2}\sum\limits_{i=1}^{D}\log \sigma _{X_{i}}^{2}-\frac{1}{2}\left[ \sum\limits_{C\in 
\mathcal{C}}\sum\limits_{y_{k}\in \mathcal{Y}}p\left( y_{k}\right) \log
\left\vert K_{X_{C}}\right\vert \right.  \\
&&\left. -\sum\limits_{S\in \mathcal{S}}\left( \nu _{S}-1\right)
\sum\limits_{y_{k}\in \mathcal{Y}}p\left( y_{k}\right) \log \left\vert
K_{X_{S|y_{k}}}\right\vert \right] 
\end{eqnarray*}
\end{proof}

\begin{comment}
\item Now we discuss the more flexible case where we have no restriction on the distributions of the explanatory variables. We suppose \textit{the joint dependence structure to be described by a Gauss copula}. In this case we also make the transformation steps described in subsection \ref{subsec:GNB gauss copula}.
We can formulate the following theorem:
\begin{theorem}
    The weight of a $k$-order GNB structure is given  as follows:
\begin{eqnarray*}
W &=&-\sum\limits_{C\in 
\mathcal{C}}\sum\limits_{y_{k}\in \mathcal{Y}}p\left( y_{k}\right) \log
\left\vert K_{Z_{C}}\right\vert  +\sum\limits_{S\in \mathcal{S}}\left( \nu _{S}-1\right)
\sum\limits_{y_{k}\in \mathcal{Y}}p\left( y_{k}\right) \log \left\vert
K_{Z_{S|y_{k}}}\right\vert 
\end{eqnarray*}
\end{theorem}
\begin{proof}
We apply the same methodology to the transformed data $\mathbf{Z}$ as in the case of the joint Gaussian distribution and obtain:
\begin{eqnarray*}
    W &=&\sum\limits_{i=1}^{d}\log \sigma _{Z_{i}}^{2}-\left[ \sum\limits_{C\in 
\mathcal{C}}\sum\limits_{y_{k}\in \mathcal{Y}}p\left( y_{k}\right) \log
\left\vert K_{Z_{C}}\right\vert \right.  \\
&&\left. -\sum\limits_{S\in \mathcal{S}}\left( \nu _{S}-1\right)
\sum\limits_{y_{k}\in \mathcal{Y}}p\left( y_{k}\right) \log \left\vert
K_{Z_{S|y_{k}}}\right\vert \right] 
\end{eqnarray*}
$\sum\limits_{i=1}^{d}\log \sigma _{Z_{i}}^{2}=0$ since all $Z_i$ are standard normal distributions with unit variance. It directly follows the formula for calculating weight.
\end{proof}
\end{comment}

\textbf{Discussion on finding $k$-GNB structures $k\geq3$}
We approach this problem from two aspects. 
Can we prove a similar result as in Section \ref{sec:combinatorial background} related to the matroid theory?

The answer is no. In this case, the indices of the features do not form a spanning tree; The independence set would contain a forest of $k$-th order cherry trees, $k\geq 3$. Consequently, the augmentation property \ref{matroid} no longer holds. This can be shown straightforwardly by constructing a counterexample. Let us consider two 3 order cherry trees.
\begin{eqnarray*}
&&(1,2,3)-[1,2]-(1,2,4)-[1,2]-(1,2,5) \\
&&(1,4,5)-[1,4]-(1,3,4)-[3,4]-(2,3,4)-[3,4]-(3,4,6)
\end{eqnarray*}
The edges consist now of 3 elements. The second cherry tree is contains more edges than the first one. Neither of the clusters of the second one van be added to the first one such that we have a valid cherry tree. Any of the clusters which contains indices already present in the first cherry tree are pointless, the triplet $(3,4,6)$ contains the index $6$, which is not in the first tree, however this can not be augmented to the first cherry tree because $[3,4]$ is not a potential separator (there is no cluster in the first tree which contains them).

The second aspect concerns the weight of the approximation. 
In Formula \ref{weight of cherry}, it can be observed that the weight of the approximation depends on the marginal probability distributions with their indices belonging to the clusters and separators.
 All information contents of subsets $k$ elements of the form $(Y,X_ {i_1},\dots,X_{i_{k-1}})$ and of subsets of $k-1$ elements of the form $(Y,X_{i_1},\dots,X_{i_{k-2}})$, with $i_1,\dots,i_{k-1} \in D$, will be calculated based on Formulas \ref{IC_Gauss} and \ref{IS_Gauss} then we aim to discover the best fitting $k$-GNB structure.
\begin{comment}
These information contents can be calculated. 
The algorithm starts from the element $(Y,X_{i_1},\dots,X_{i_{k-1}})$ with the highest information and continues by adding a new cluster greedily.
\end{comment}
There can be introduced greedy algorithms for doing this, but for cases $k>3$ there is no guarantee, for finding it in polynomial time. 

\section{The classification procedure and model reduction}
\label{Sec:Classification}
The first subsection presents the classification method using all available features. The second subsection then introduces a model reduction approach and presents the corresponding classification results based on the reduced model.
\subsection{All features classification procedure }
\label{Subsec:classification-all features}
We discuss first the classification in the case when all the features are considered. Based on the training data, we find a GNB structure by applying Algorithm \ref{alg:gnb-cont}. As a result, we obtain a set of clusters and a set of separators that define the GNB structure. With these components as inputs, we describe the main steps involved in the classification process, see Figure \ref{fig:classification_procedure}.
Based on the approximation Formula \ref{eq:Cherry_tree} the log-likelihood of a new point $\mathbf{x^\star}$ in a class $y\in\{y_1,\dots,y_m\}$ can be calculated.

\begin{description} 

\item[Input:] The set of clusters and the set of separators which contain the index 0 corresponding the $Y$ target variable:
$(0,{i_{k}},{j_{k}) \in \mathcal{C}}$ and $( 0,{i_{k}}) \in \mathcal{S}$ and the corresponding marginal probability distributions: $( Y,X_{i_{k}},X_{j_{k}})$ and $(Y,X_{i_{k}})$; A new datapoint $\mathbf{x}=(x^*_1,\dots, x^*_d)$
\item[Log-likelihood evaluation] For each $y_k\in\{y_1,\dots,y_m\}$ the log-likelihood of a new datapoints $\mathbf{x^*}$ is calculated  by using GNB approximation ${p_{GNB}}( y,\mathbf{x^*})$ :
\begin{align*}
    \log \frac{\prod\limits_{(0,i,j)\in \mathcal{C}}p\left( y_k,x^*_{i},x^*_{j}\right) }{\prod\limits_{(0,s)\in \mathcal{S}}p\left(
y_k,x^*_s\right)
^{v_{S}-1}}=\sum\limits_{\left( 0,{i},j\right) \in \mathcal{C}}{\log p\left( y_k,x^*_{i},x^*_{j}\right) }- \\-{\sum\limits_{\left( 0,s\right) \in \mathcal{S}} {(v_{s}-1)}\log p\left(
y_k,x^*_s\right)}
\end{align*}
\item[Output:] The class  $y_k$ with the maximal log-likelihood; Diagnostics: Graph representations on the good and bad classified points projected to the feature pairs selected.
\end{description}
For the Diagnostics see an illustration on the Iris data in \ref{Apendix_diagnostics based on pairs}: In Figures \ref{Gauss_Iris}, \ref{Gauss copula Iris}, \ref{KDE Iris}, one can observe respectively the pairs in the Gaussian case, Gauss copula with KDE marginals, and the most general case where the copula and the marginals are modeled by KDEs.  

\begin{figure}[h]
    \centering
    \includegraphics[scale=0.50]{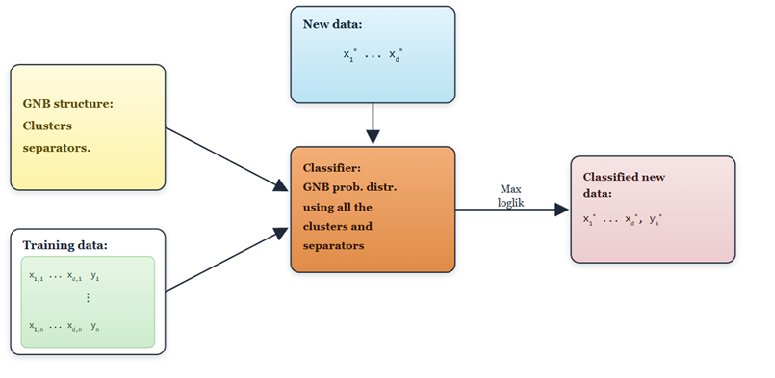}
    \caption{The classification procedure has as input: the structure, the training data. The log-likelihood uses the restricted marginal p.d. to each of the classes separately ; the classifier chooses the class with maximal log-likelihood. }
   \label{fig:classification_procedure} 
\end{figure}

We emphasize here that we used the default parameters for kernel density estimation, for a general fitting. Although for practitioners interested in a given problem it is worth to finetuning these parameters. Such setting can improve essential the fitting to the training data, see the example in Appendix \ref{KDE tuned Iris}. However this may lead to overfitting on the test data, therefore for finetuning a good idea is to use a validation set. 

\subsection{Model reduction and classification based on the reduced model}
\label{subsec:model reduction}
When the number of explanatory attributes is very high, most of the classification methods use feature selection beforehand or regularization in cases like LR and SVM to reduce the number of explanatory variables. 

We propose a method that selects one reduced model from a sequence of reduced models. These reduced models are evaluated on the training data. It would also be possible to apply this approach to a validation dataset; however, we chose to rely exclusively on the training data here to ensure comparability with LR and SVM, for which only training data was used.

The reduction procedure has as input the ordered set of clusters obtained from the Structure Learning Algorithm \ref{alg:gnb-cont}.

The smallest reduced model has one cluster. Each new reduced model is obtained from the former one by adding to it a new cluster from the ordered sequence. It is important to point out that by adding greedily new clusters we may obtain a forest as an intermediate structure which is a set of GNB trees, therefore we need to introduce the following concept. The workflow of the method can be seen in Figure \label{fig:reduced model learning}.

\begin{definition}
\label{def:GNB_forest}
We call \textbf{GNB forest} on a set $D'\subset D$ a structure which has the property that all of the indices in $D'$ are linked to $0$ (which stands for the classification variable $Y$) and a forest is defined on the set of indices in $D'$.
\end{definition}
Let us denote by $\mathcal{T}_{D'}$ the set of disjunct trees defined on $D'\subset D$. 

\begin{theorem}
    The probability distribution corresponding to an intermediate structure constructed from the sequence given by Algorithm \ref{alg:gnb-cont} is a GNB forest, and the probability distribution assigned is given by
    \begin{equation}
    \label{eq: prob reduced model}
p(Y,\mathbf{X}_{D^{\prime }})=\prod\limits_{T_{k}\in \mathcal{T}}\frac{%
\prod\limits_{(0,i_{T_{k}},j_{T_{k}})\in \mathcal{C}%
_{T_{k}}}p(Y,X_{i_{T_{k}}},X_{j_{T_{k}}})}{\prod\limits_{(0,s_{T_{k}})\in 
\mathcal{S}_{T_{k}}}p(Y,X_{s_{T_{k}}})^{\left( v_{s_{T_{k}}}-1\right) }}%
\cdot \frac{1}{p(Y)^{{\left\vert \mathcal{T}\right\vert }-1}}
\end{equation}
where $\mathcal{T}_{D'}$ denotes the forest constructed by Kruskal's algorithm on the intermediate set of vertices $D'$
\end{theorem}

\begin{proof}
  In each step of Algorithm \ref{alg:gnb-cont} Kruskal's Algorithm adds a new edge. This can be achieved in two ways: by connecting a new vertex to the existing forest or by adding two new vertices connected by a new edge. In this way in each step a new cluster is added to the GNB forest (Definition \ref{def:GNB_forest}).

 Two clusters (triplets of the form $(0,i,j)$) are connected either by only the vertex $0$ or by two vertices one of them is $0$ the other one contained by $D'$, see Figure \ref{fig:GNB_forest}.

When a subset of vertices in $D'$ form a tree, the clusters are connected  to each other by two vertices (all separators contain two vertices), so this defines sub GNB structure (corresponding to each $T_k \in \mathcal{T}$ tree); The probability distribution assigned to a sub-GNB structure corresponding to a $T_k \in \mathcal{T}$ tree is 
    \begin{equation}
    \label{GNB_Tk}
P(\mathbf{X}_{T_k})=\frac{\prod\limits_{(0,i_{T_{k}},j_{T_{k}})\in \mathcal{C}%
_{T_{k}}}p(Y,X_{i_{T_{k}}},X_{j_{T_{k}}})}{{\prod\limits_{(0,s_{T_{k}})\in 
\mathcal{S}_{T_{k}}}p(Y,X_{s_{T_{k}}})^{\left( v_{s_{T_{k}}}-1\right) }}} 
    \end{equation}
where $\mathcal{C}_{T_{k}}$ respectively $\mathcal{S}_{T_{k}}$ denotes the set of cluster and separators of the sub-GNB structure.

Let us consider the case when between vertices in $D'$ their is no path, i.e. the indices belong to different trees. In this case the GNB forest with the indices of the explanatory variables in $D'$ may contain multiple sub-GNB structures, having in common only $Y$, see Figure \ref{fig:GNB_forest}. The probability distribution assigned is given by the GNB structures involved.
The number of the trees is denoted by $ \left\vert \mathcal{T}\right\vert $ , where  $\left\vert \mathcal{T}\right\vert \geq1 $. 
A GNB structure corresponding to different trees is illustrated in Figure \ref{fig:GNB_forest}. It is easy to see that the intersection between the sub GNB structures is only $Y$ which is an one-element separator. The number of one-element separators of this type are exactly $\left\vert \mathcal{T}\right\vert-1$.  

From this observation and Formula \ref{GNB_Tk} it follows immediately the formula of the probability distribution \ref{eq: prob reduced model} given in the theorem.
\end{proof}
Let us note that if we have exactly one tree then $\frac{1}{P(Y)^{\left\vert \mathcal{T}\right\vert -1}}$ equals to 1.
\begin{figure}[h]
    \centering
    \includegraphics[scale=0.20]{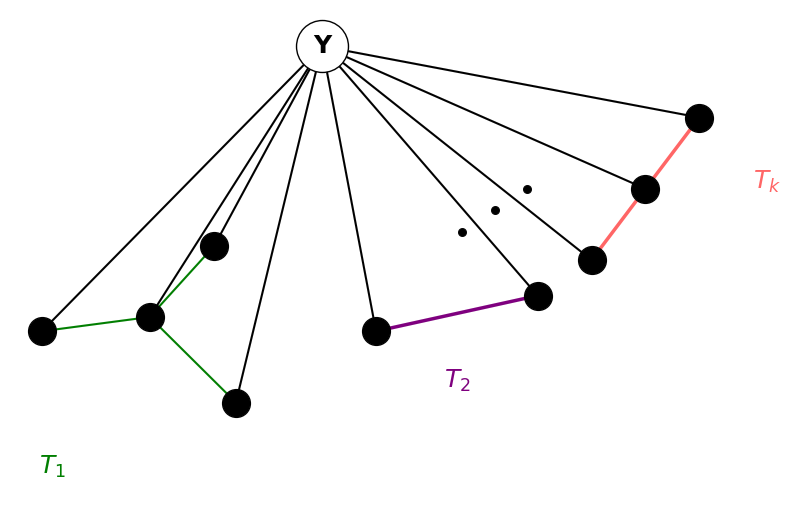}
    \caption{This illustrates an intermediate step of \ref{alg:gnb-cont} presenting a possible reduced model, when not all the edges between the indices (vertices) are added so far. $T_1, T_2,..., T_k$ are trees.}
    \label{fig:GNB_forest}
\end{figure}

The log-likelihood of a sample vector $(\mathbf{x^*},y_k)$ is calculated based on the reduced model \ref{eq: prob reduced model} restricted to the GNB forest containing the attributes with indices in $D^{\prime }\subset D$ by the following formula:
\begin{equation}
\label{eq:loglik_red}
\begin{array}{c}
\log p(\mathbf{x^*}_{D^{\prime }},y_k)=\sum\limits_{T_{k}\in \mathcal{T}}\left[
\sum\limits_{(0,i_{T_{k}},j_{T_{k}})\in \mathcal{C}_{T_{k}}}\log
p(y_k,x^*_{i_{T_{k}}},x^*_{j_{T_{k}}})\right.  \\ 
\left. -\sum\limits_{(0,s_{T_{k}})\in \mathcal{S}_{T_{k}}}\left(
v_{s_{T_{k}}}-1\right) \log p(y_k,x^*_{s_{T_{k}}})\right] -\left\vert \mathcal{T}%
\right\vert \log p(y_k)%
\end{array}%
\end{equation}

We assign to the number of clusters/features used by all the reduced models the $f1$-score calculated on the training data, see the graphs in \ref{Appendix:model_reduction_graphs}. One can chose automatically the number of clusters having the highest average $f1$-scores over a number of splits (in our experiments this is chosen to be 10).  The user can also decide based of the graphics weather a slight improvement is an enough reason to use more clusters. Obviously this function is not a monotonic one.

The selection of the reduce model is illustrated in Figure \ref{fig:reduced_model_learning}.
\begin{figure}[h] 
    \centering
    \includegraphics[scale=0.30]{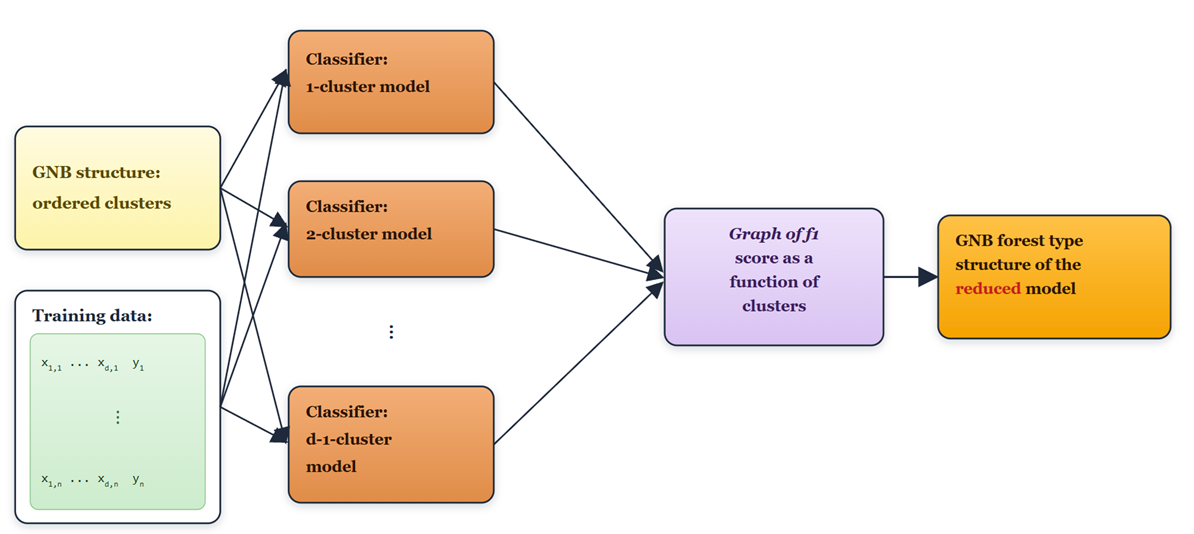}
    \caption{Based on the output of Algorithm \ref{alg:gnb-cont} a sequence of ordered cluster is obtained (yellow box). Using the data the marginals corresponding to the clusters and separators are determined. Based on this we obtain a 1-cluster classifier, 2-cluster classifier and so on. These are tested on the training data (in our experiments). The output is a graph of $f1$ score as function of the number of clusters. The optimal reduced model is chosen to have the smallest number of clusters where $f1$ reaches the maximum.}  
    \label{fig:reduced_model_learning}
\end{figure}

The classification procedure based on the reduced model is described as follows:
\begin{description} 
\item[Input:] The set of indices $D'$; the GNB forest on $D'$: the set of clusters, set of separators  with their multiplicity; the corresponding marginal probability distributions: $( Y,X_{i_{k}},X_{j_{k}})$ and $(Y,X_{i_{k}})$; 
A new datapoint $\mathbf{x}=(x^*_1,\dots, x^*_d)$
\item[Log-likelihood evaluation] 
For each value of $y=y_1,\dots,y_m$ the log-likelihood of a new datapoint $\mathbf{x^*}$ is calculated :
$\log p(\mathbf{x^*}_{D^{\prime }},y_k)$ \hspace{3em} $\blacktriangleright$Theorem:  \ref{eq:loglik_red}
\item[Output:] The class  $y_k\in \left\{y_1,\ldots ,y_m\right\} $ with the maximal log-likelihood.
\end{description}

\section{Numerical Results}
\label{Sec:Numerical results}
The introduced algorithms are designed for the classification task when attributes are continuous, without using discretizations. We will use evaluation metrics specific to this problem: accuracy, precision, recall, $f1$-score, see \ref{Appendix:Accuracy}. In the case of more than two classes we used weighted scores.
Our algorithms (GNB-Gauss, GNB G-cop, GNB KDE) respectively the reduced models (GNBGauss*, GNB G-cop*, GNB KDE*) were implemented in Python, and will be publicly available for testing.

\subsection{Data preparation methods and experiment description}
\label{dataprep} 
The GNB algorithms introduced in this paper are developed for data characterized by continuous features, without using discretizations. However all data was scaled by min-max scaler, and when two variables correlate more in absolute value than 0.99 we delete one of them.
 To have a general setting and to be able to apply the algorithms in the same conditions, we choose for all the compared algorithms the same random $15\%$ test set selection, having the same random ten splits (seeds: form 42 to 52). 

Relative to our methods we conducted two kinds of experiments. One containing all of the explanatory variables, the second one by using "the best" reduced model. The reduction of the model is made automatically as described in Subsection \ref{subsec:model reduction} characterizing each of the reduced models by the mean $f1$ score on the training data over the $10$ runs. 

In order to compare our algorithms to other glass-box algorithms we run the logistic regression (LR) and Support Vector Machine (SVM) algorithms under different regularization parameter ($C=1000$, $C=1$, $ C=0.01$) where $C$ stands for regularizing parameter. 

\subsection{Data description}

We give in the following a short description of the datasets in the table \ref{Description of Datasets}.
\bgroup
\def\arraystretch{0.75}
\begin{table}
\centering
\begin{tabular}{ |c|c|c|c|c| }
 \hline
 & Num. of  & Num. of & Num. of \\ 
 Dataset & attributes  & classes & entries  \\ 
  & (columns) & attributes& (rows) \\ 
 \hline
 Abalone \cite{abalone_1} & 8 & 3 & 4180 \\
Body fat \cite{penrose1985generalized} & 13 & 3 & 252 \\
 WBCD \cite{wdbc_link} & 30 & 2 & 569 \\
 Connect \cite{connectionist_bench}  & 60 & 2 & 208 \\
 Iris \cite{iris_53} & 4 & 3 & 150\\  
 Wine \cite{wine_109} & 14 & 3 & 178 \\
 Wine quality \cite{wine_quality_186} & 12  & 2 & 4898 \\
 \hline
\end{tabular}
\caption{Short description of the datasets we used in this paper. The usage of the datasets is described in \label{Description of Datasets}}
\end{table}
\egroup

In the \textit{Diagnostic Wisconsin Breast Cancer Database (WBCD)} \cite{wdbc_link} the aim is to classify a tumor as benign or malignant. The data is slightly unbalanced 0.63 are benign (B) and 0.37 are malign (M).
The \textit{Abalone} data \cite{abalone_1} the problem is predicting the sex of abalone shell from physical measurements.
The \textit{Body fat} dataset \cite{penrose1985generalized} lists estimates of the percentage of body fat determined by underwater
weighing and various body circumference measurements for 252 men. The target variable has three classes. The \textit{Iris} dataset \cite{iris_53} is a well known dataset with 50 samples from each of the three species. The task is to classify the flowers based on four features. The \textit{Connectionist data} \cite{connectionist_bench} contains 111 patterns obtained by bouncing sonar signals off a metal cylinder at various angles and under various conditions. Each pattern is a set of 60 numbers in the range 0.0 to 1.0.  Each number represents the energy within a particular frequency band, integrated over a certain period of time.The label associated with each record contains the letter "R" if the object is a rock and "M" if it is a mine (metal cylinder).

The \textit{Parkinson's Disease (Parkinson)}  data set \cite{parkinson_link} contains biomedical voice measurements.
The aim is to discriminate healthy people from those with PD.
There are two wine related datasets: \textit{Wine data} \cite{wine_quality_186} which is related to red variants of the Portuguese "Vinho Verde" wine. The dataset describes the amount of various chemicals present in wine and their effect on it's quality. The datasets can be viewed as classification task. The classes are ordered and not balanced (e.g. there are much more normal wines than excellent or poor ones). The original task was to predict the quality wine score (0-10) using the given data. This problem was reformulated and the predict good quality wines with score greater or equal than $7$ and average quality wines with their score smaller than $7$.
The second one is \textit{Wine quality} data \cite{wine_109} which contains results of a chemical analysis of wines grown in the same region in Italy but derived from three different cultivars. The analysis determined the quantities of 13 constituents found in each of the three types of wines. 

\subsection{Results and discussion}
The results are presented in the Tables: \ref{tab:Abalone_Body fat}, \ref{tab:WBCD_Connect}, \ref{tab:Iris_Parkinson} and \ref{tab:Wine}. We use for each dataset accuracy, precision, recall and $f1$ scores. When there were more than two classes we used their weighted average, implemented in Python. This way the problems which occur because of unbalanced data can also be observed and compared. 

For each of the datasets the tables are presented in two rows. The classification Cases where all features were utilized in the classification are highlighted in yellow. For Logistic regression and SVM we used high value for the inverse of the regularization term $C=1000$ in order to avoid regularization and use all the features.
In the white columns, besides Gaussian Naive Bayes (GaussNB), Kernel Naive Bayes (KernelNB) we apply Logistic Regression (LR) and Support Vector Machine (SVM) with different regularization coefficients. 

Our newly introduced methods: Gaussian GNB (GNBGauss), Gauss copula with kernel estimated marginals (GNB G-copula) and the most general method where either the joint probability density either the copula an the uni-variate marginals are estimated by kernel densities(GNB KDE). Moreover, we introduced the model reduction possibility in subsection \ref{subsec:model reduction}. The automatically reduced models were denoted by appending a star to their names.

\vspace{3mm}
In Table \ref{tab:Abalone_Body fat}, on \textit{Abalone} dataset: The "all features" winner based on $f1$score is LR (C=1000) which is also the overall winner; Among the new GNB methods, the highest $f1$ score was achieved by GNB KDE ($96\%$ of the winner $f1$score).  On \textit{Body fat} dataset "all feature" winner and overall winner is  GNB G-Cop*.  Default LR(C=1) and SVM(C=1) are dominated by their version with C=1000 (no regularization). SVM (C=1000) was the closest to the winner ( $99\%$ of the winner $f1$ score).

In Table \ref{tab:WBCD_Connect}, on \textit{WBCD} dataset SVM (C=1000) is the "all features" winner, whereas the reduced model winner is LR (C=1); Among our methods The GNB G-cop was the best ($97.5\%$ of the winner $f1$ score).
On \textit{Connect} dataset GNB KDE model is the "all feature" winner while the reduced  GNB KDE* model is the overall best. 

In Table \ref{tab:Iris_Parkinson}, on \textit{Iris} dataset GNB G-cop was "all features" winner, while the reduced model GNB G-cop* also the overall winner; on \textit{Parkinson} dataset SVM(C=1) is the the overall winner. Among our methods the winner is GNB KDE which achieves ($95.8\%$ of the winner $f1$ score)

In Table \ref{tab:Wine}, \textit{Wine} dataset the "all features" and overall winner are the GNBGauss and GNB G-cop methods, among the other methods GaussNB performed also very well. In \textit{Wine quality} dataset the "all features" and overall winner are the GNB KDE  method, among the reduced methods GNB G-cop was the best one. Interestingly to note that LR and SVM methods classified very poorly. On the other hand GaussNB and K.NB performed better than LR and SVM with different parameter setting. 

For a good overview,  we calculate for each of the datasets the relative $f1$ scores as a fraction obtained by dividing the $f1$ score obtained by each of the methods with the highest $f1$ score achieved on the given dataset (mean over 10 random sampling). Then we calculate the average of these relative $f1$ scores for each of the methods over all of the datasets. The results are presented in Figure \ref{összehasonlítás}.
\begin{figure}[h]
    \centering
    \includegraphics[scale=0.32]{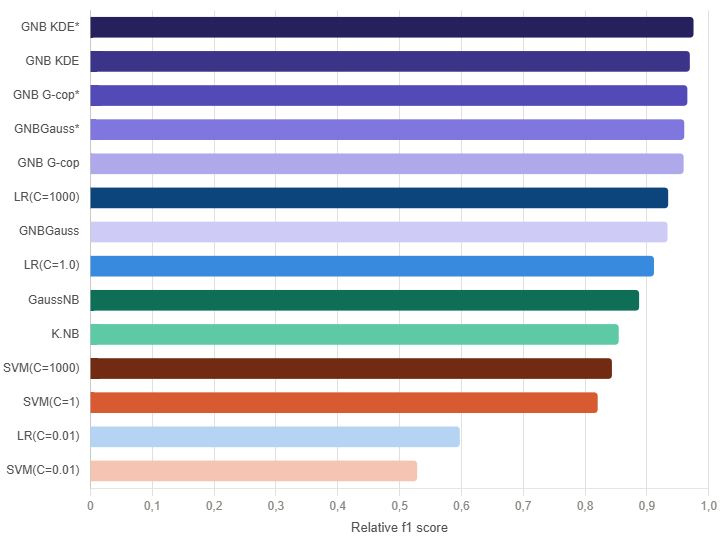}
    \caption{Average of the relative $f1$ scores for each of the methods over all datasets}   
    \label{összehasonlítás}
\end{figure}
We mention here that in some of the cases the automatic reduced model chores algorithm also decided to use all the explanatory variable. Also it is worth to draw the attention that we use here Gaussian kernels without tuning their parameters, therefore we are sure that for a given problem of interest it is worth to finetuning the parameters in order to get better results. We illustrate this on the Iris data, see Figure \ref{KDE Iris}. 

In conclusion, the newly introduced methods are competitive with established glass-box classification methods, such as logistic regression and support vector machines. In many cases, they even outperformed these traditional glass-box approaches. Also important to see that this methods are capable to use as many features as possible by avoiding redundancy since our models exploit the conditional independencies. 
We highlight here that our reduction method \ref{Appendix:model_reduction_graphs} outputs graphs based on which human decision maker may intervene in the process. These graphs represents how $f1$ score (calculated on the training data) depends on the number of clusters involved on each train test split and the graph obtained by averaging over 10 random splits. Based on these graphs the user may decide on how many clusters to use in the classification task. However in all our experiments this was done automatically, by choosing the number of clusters which achieved the highest $f1$ score on the average of the 10 runs.

\begin{landscape}

\begin{table}[]
\caption{The table contains the mean accuracy scores (rows) for different methods (columns) over 10 random splits on the Abalone and Body fat test data. We highlighted in yellow when all features were used, and by star our reduced models.}
\label{tab:Abalone_Body fat}
\begin{tabular}{cc
>{\columncolor[HTML]{FFFFCC}}c 
>{\columncolor[HTML]{FFFFCC}}c c
>{\columncolor[HTML]{FFFFFF}}c c
>{\columncolor[HTML]{FFFFFF}}c }
\cellcolor[HTML]{DAF2D0}\textbf{Abalone} & \cellcolor[HTML]{F7C7AC}GaussNB & \cellcolor[HTML]{F7C7AC}K.NB & \cellcolor[HTML]{F7C7AC}LR(C=1000) & \cellcolor[HTML]{F7C7AC}LR(C=1.0) & \cellcolor[HTML]{F7C7AC}LR(C=0,01) & \cellcolor[HTML]{F7C7AC}SVM(C=1000) & \cellcolor[HTML]{F7C7AC}SVM(C=1) \\
accuracy & \cellcolor[HTML]{FFFFCC}0.5270 & 0.5118 & 0.5593 & \cellcolor[HTML]{FFFFFF}0.5577 & 0.5308 & \cellcolor[HTML]{FFFFCC}0.5560 & 0.5451 \\
precision & \cellcolor[HTML]{FFFFCC}0.5127 & 0.3319 & 0.5448 & \cellcolor[HTML]{FFFFFF}0.5443 & 0.5421 & \cellcolor[HTML]{FFFFCC}0.5439 & 0.6938 \\
recall & \cellcolor[HTML]{FFFFCC}0.5270 & 0.5118 & 0.5593 & \cellcolor[HTML]{FFFFFF}0.5577 & 0.5308 & \cellcolor[HTML]{FFFFCC}0.5560 & 0.5451 \\
f1-score & \cellcolor[HTML]{FFFFCC}0.4783 & 0.4012 & \textbf{0.5459} & \cellcolor[HTML]{FFFFFF}0.5420 & 0.4408 & \cellcolor[HTML]{FFFFCC}0.5235 & 0.4475 \\
 & \cellcolor[HTML]{F7C7AC}SVM(C=0.01) & \cellcolor[HTML]{F7C7AC}GNBGauss & \cellcolor[HTML]{F7C7AC}GNB G-cop & \cellcolor[HTML]{F7C7AC}GNB KDE & \cellcolor[HTML]{F7C7AC}GNBGauss*(4/6) & \cellcolor[HTML]{F7C7AC}GNB G-cop* (4/6)) & \cellcolor[HTML]{F7C7AC}GNB KDE*(6/6) \\
accuracy & \cellcolor[HTML]{FFFFFF}0.5278 & 0.5271 & 0.5204 & \cellcolor[HTML]{FFFFCC}0.5376 & 0.5432 & \cellcolor[HTML]{FFFFFF}0.5313 & 0.5376 \\
precision & \cellcolor[HTML]{FFFFFF}0.3751 & 0.5113 & 0.5015 & \cellcolor[HTML]{FFFFCC}0.5207 & 0.5346 & \cellcolor[HTML]{FFFFFF}0.5133 & 0.5207 \\
recall & \cellcolor[HTML]{FFFFFF}0.5278 & 0.5271 & 0.5204 & \cellcolor[HTML]{FFFFCC}0.5376 & 0.5432 & \cellcolor[HTML]{FFFFFF}0.5313 & 0.5376 \\
f1-score & \cellcolor[HTML]{FFFFFF}0.4344 & 0.5002 & 0.4982 & \cellcolor[HTML]{FFFFCC}0.5238 & 0.5165 & \cellcolor[HTML]{FFFFFF}0.5077 & 0.5238 \\
\cellcolor[HTML]{DAF2D0}\textbf{Body fat} & \cellcolor[HTML]{F7C7AC}GaussNB & \cellcolor[HTML]{F7C7AC}K,NB & \cellcolor[HTML]{F7C7AC}LR(C=1000) & \cellcolor[HTML]{F7C7AC}LR(C=1,0) & \cellcolor[HTML]{F7C7AC}LR(C=0,01) & \cellcolor[HTML]{F7C7AC}SVM(C=1000) & \cellcolor[HTML]{F7C7AC}SVM(C=1) \\
accuracy & \cellcolor[HTML]{FFFFCC}0.6553 & 0.6289 & 0.7947 & \cellcolor[HTML]{FFFFFF}0.7500 & 0.5789 & \cellcolor[HTML]{FFFFCC}0.8211 & 0.7500 \\
precision & \cellcolor[HTML]{FFFFCC}0.6811 & 0.6601 & 0.7925 & \cellcolor[HTML]{FFFFFF}0.7683 & 0.3352 & \cellcolor[HTML]{FFFFCC}0.8194 & 0.7697 \\
recall & \cellcolor[HTML]{FFFFCC}0.6553 & 0.6289 & 0.7947 & \cellcolor[HTML]{FFFFFF}0.7500 & 0.5789 & \cellcolor[HTML]{FFFFCC}0.8211 & 0.7500 \\
f1-score & \cellcolor[HTML]{FFFFCC}0.6576 & 0.6290 & 0.7935 & \cellcolor[HTML]{FFFFFF}0.7273 & 0.4246 & \cellcolor[HTML]{FFFFCC}0.8202 & 0.7240 \\
 & \cellcolor[HTML]{F7C7AC}SVM(C=0,01) & \cellcolor[HTML]{F7C7AC}GNBGauss & \cellcolor[HTML]{F7C7AC}GNB G-cop & \cellcolor[HTML]{F7C7AC}GNB KDE & \cellcolor[HTML]{F7C7AC}GNBGauss*(2/12) & \cellcolor[HTML]{F7C7AC}GNB G-cop* (10/12) & \cellcolor[HTML]{F7C7AC}GNB KDE* (9/12) \\
accuracy & \cellcolor[HTML]{FFFFFF}0.5789 & 0.7579 & 0.8105 & \cellcolor[HTML]{FFFFCC}0.7579 & 0.8132 & \cellcolor[HTML]{FFFFFF}0.8316 & 0.7868 \\
precision & \cellcolor[HTML]{FFFFFF}0.3352 & 0.7534 & 0.8096 & \cellcolor[HTML]{FFFFCC}0.7506 & 0.8100 & \cellcolor[HTML]{FFFFFF}0.8302 & 0.7818 \\
recall & \cellcolor[HTML]{FFFFFF}0.5789 & 0.7579 & 0.8105 & \cellcolor[HTML]{FFFFCC}0.7579 & 0.8132 & \cellcolor[HTML]{FFFFFF}0.8316 & 0.7868 \\
f1-score & \cellcolor[HTML]{FFFFFF}0.4246 & 0.7503 & 0.8068 & \cellcolor[HTML]{FFFFCC}0.7515 & 0.8107 & \cellcolor[HTML]{FFFFFF}\textbf{0.8280} & 0.7813
\end{tabular}
\end{table}
\end{landscape}
\newpage

\begin{landscape}

\begin{table}[]
\caption{The table contains the mean accuracy scores (rows) for different methods (columns) over 10 random splits on the WBCD and Connectionist fat test data. We highlighted in yellow when all features were used, and by star our reduced models.}
\label{tab:WBCD_Connect}
\begin{tabular}{cc
>{\columncolor[HTML]{FFFFCC}}c 
>{\columncolor[HTML]{FFFFCC}}c c
>{\columncolor[HTML]{FFFFFF}}c c
>{\columncolor[HTML]{FFFFFF}}c }
\cellcolor[HTML]{DAF2D0}\textbf{WBCD} & \cellcolor[HTML]{F7C7AC}GaussNB & \cellcolor[HTML]{F7C7AC}K.NB & \cellcolor[HTML]{F7C7AC}LR(C=1000) & \cellcolor[HTML]{F7C7AC}LR(C=1.0) & \cellcolor[HTML]{F7C7AC}LR(C=0,01) & \cellcolor[HTML]{F7C7AC}SVM(C=1000) & \cellcolor[HTML]{F7C7AC}SVM(C=1) \\
accuracy & \cellcolor[HTML]{FFFFCC}0.9360 & 0.9151 & 0.9663 & \cellcolor[HTML]{FFFFFF}0.9698 & 0.7907 & \cellcolor[HTML]{FFFFCC}0.9651 & 0.9779 \\
precision & \cellcolor[HTML]{FFFFCC}0.9316 & 0.9881 & 0.9709 & \cellcolor[HTML]{FFFFFF}0.9933 & 1.0000 & \cellcolor[HTML]{FFFFCC}0.9708 & 0.9902 \\
recall & \cellcolor[HTML]{FFFFCC}0.8938 & 0.7813 & 0.9375 & \cellcolor[HTML]{FFFFFF}0.9250 & 0.4375 & \cellcolor[HTML]{FFFFCC}0.9344 & 0.950000 \\
f1-score & \cellcolor[HTML]{FFFFCC}0.9123 & 0.8726 & \textbf{0.9539} & \cellcolor[HTML]{FFFFFF}0.9579 & 0.6087 & \cellcolor[HTML]{FFFFCC}0.9522 & \textbf{0.9697} \\
 & \cellcolor[HTML]{F7C7AC}SVM(C=0.01) & \cellcolor[HTML]{F7C7AC}GNBGauss & \cellcolor[HTML]{F7C7AC}GNB G-cop & \cellcolor[HTML]{F7C7AC}GNB KDE & \cellcolor[HTML]{F7C7AC}GNBGauss* (29/29) & \cellcolor[HTML]{F7C7AC}GNB G-cop* (29/29) & \cellcolor[HTML]{F7C7AC}GNB KDE* (28/29) \\
accuracy & \cellcolor[HTML]{FFFFFF}0.8756 & 0.9570 & 0.9593 & \cellcolor[HTML]{FFFFCC}0.9488 & 0.9570 & \cellcolor[HTML]{FFFFFF}0.9593 & 0.9477 \\
precision & \cellcolor[HTML]{FFFFFF}1.0000 & 0.9408 & 0.9439 & \cellcolor[HTML]{FFFFCC}0.9423 & 0.9408 & \cellcolor[HTML]{FFFFFF}0.9439 & 0.9421 \\
recall & \cellcolor[HTML]{FFFFFF}0.665625 & 0.9438 & 0.9469 & \cellcolor[HTML]{FFFFCC}0.9188 & 0.9438 & \cellcolor[HTML]{FFFFFF}0.9469 & 0.9156 \\
f1-score & \cellcolor[HTML]{FFFFFF}0.7993 & 0.9423 & 0.9454 & \cellcolor[HTML]{FFFFCC}0.9304 & 0.9423 & \cellcolor[HTML]{FFFFFF}0.9454 & 0.9287 \\
\cellcolor[HTML]{DAF2D0}\textbf{Connect} & \cellcolor[HTML]{F7C7AC}GaussNB & \cellcolor[HTML]{F7C7AC}K.NB & \cellcolor[HTML]{F7C7AC}LR(C=1000) & \cellcolor[HTML]{F7C7AC}LR(C=1.0) & \cellcolor[HTML]{F7C7AC}LR(C=0,01) & \cellcolor[HTML]{F7C7AC}SVM(C=1000) & \cellcolor[HTML]{F7C7AC}SVM(C=1) \\
accuracy & \cellcolor[HTML]{FFFFCC}0.6813 & 0.7000 & 0.7719 & \cellcolor[HTML]{FFFFFF}0.7969 & 0.6906 & \cellcolor[HTML]{FFFFCC}0.7656 & 0.7938 \\
precision & \cellcolor[HTML]{FFFFCC}0.7656 & 0.8304 & 0.7680 & \cellcolor[HTML]{FFFFFF}0.8000 & 0.6320 & \cellcolor[HTML]{FFFFCC}0.7654 & 0.7889 \\
recall & \cellcolor[HTML]{FFFFCC}0.5765 & 0.5471 & 0.8176 & \cellcolor[HTML]{FFFFFF}0.8235 & 1.0000 & \cellcolor[HTML]{FFFFCC}0.8059 & 0.8353 \\
f1-score & \cellcolor[HTML]{FFFFCC}0.6577 & 0.6596 & 0.7920 & \cellcolor[HTML]{FFFFFF}0.8116 & 0.7745 & \cellcolor[HTML]{FFFFCC}0.7851 & 0.8114 \\
 & \cellcolor[HTML]{F7C7AC}SVM(C=0.01) & \cellcolor[HTML]{F7C7AC}GNBGauss & \cellcolor[HTML]{F7C7AC}GNB G-cop & \cellcolor[HTML]{F7C7AC}GNB KDE & \cellcolor[HTML]{F7C7AC}GNBGauss* (44/59) & \cellcolor[HTML]{F7C7AC}GNB G-cop* (57/59) & \cellcolor[HTML]{F7C7AC}GNB KDE* (46/59) \\
accuracy & \cellcolor[HTML]{FFFFFF}0.5313 & 0.7563 & 0.8375 & \cellcolor[HTML]{FFFFCC}0.8563 & 0.8469 & \cellcolor[HTML]{FFFFFF}0.8469 & 0.8656 \\
precision & \cellcolor[HTML]{FFFFFF}0.5313 & 0.8151 & 0.8172 & \cellcolor[HTML]{FFFFCC}0.8407 & 0.8805 & \cellcolor[HTML]{FFFFFF}0.8235 & 0.8470 \\
recall & \cellcolor[HTML]{FFFFFF}1.0000 & 0.7000 & 0.8941 & \cellcolor[HTML]{FFFFCC}0.9000 & 0.8235 & \cellcolor[HTML]{FFFFFF}0.9059 & 0.9118 \\
f1-score & \cellcolor[HTML]{FFFFFF}0.6939 & 0.7532 & 0.8539 & \cellcolor[HTML]{FFFFCC}0.8693 & 0.8511 & \cellcolor[HTML]{FFFFFF}\textbf{0.8627} & \textbf{0.8782}
\end{tabular}
\end{table}
\end{landscape}

\begin{landscape}
\begin{table}[]
\caption{The table contains the mean accuracy scores (rows) for different methods (columns) over the same 10 random splits on the Iris and Parkinson data. We highlighted in yellow when all features were used, and by star our reduced models.}
\label{tab:Iris_Parkinson}
\begin{tabular}{cc
>{\columncolor[HTML]{FFFFCC}}c 
>{\columncolor[HTML]{FFFFCC}}c c
>{\columncolor[HTML]{FFFFFF}}c c
>{\columncolor[HTML]{FFFFFF}}c }
\cellcolor[HTML]{DAF2D0}\textbf{Iris} & \cellcolor[HTML]{F7C7AC}\textbf{GaussNB} & \cellcolor[HTML]{F7C7AC}\textbf{K.NB} & \cellcolor[HTML]{F7C7AC}\textbf{LR(C=1000)} & \cellcolor[HTML]{F7C7AC}\textbf{LR(C=1.0)} & \cellcolor[HTML]{F7C7AC}\textbf{LR(C=0,01)} & \cellcolor[HTML]{F7C7AC}\textbf{SVM(C=1000)} & \cellcolor[HTML]{F7C7AC}\textbf{SVM(C=1)} \\
accuracy & \cellcolor[HTML]{FFFFCC}0.9565 & 0.9087 & 0.9870 & \cellcolor[HTML]{FFFFFF}0.9261 & 0.7391 & \cellcolor[HTML]{FFFFCC}0.9783 & 0.9565 \\
precision & \cellcolor[HTML]{FFFFCC}0.9567 & 0.9137 & 0.9870 & \cellcolor[HTML]{FFFFFF}0.9272 & 0.8217 & \cellcolor[HTML]{FFFFCC}0.9783 & 0.9567 \\
recall & \cellcolor[HTML]{FFFFCC}0.9565 & 0.9087 & 0.9870 & \cellcolor[HTML]{FFFFFF}0.9261 & 0.7391 & \cellcolor[HTML]{FFFFCC}0.9783 & 0.9565 \\
f1-score & \cellcolor[HTML]{FFFFCC}0.9565 & 0.9082 & \textbf{0.9870} & \cellcolor[HTML]{FFFFFF}0.9260 & 0.6825 & \cellcolor[HTML]{FFFFCC}0.9783 & \textbf{0.9565} \\
 & \cellcolor[HTML]{F7C7AC}\textbf{SVM(C=0.01)} & \cellcolor[HTML]{F7C7AC}\textbf{GNBGauss} & \cellcolor[HTML]{F7C7AC}\textbf{GNB G-cop} & \cellcolor[HTML]{F7C7AC}\textbf{GNB KDE} & \cellcolor[HTML]{F7C7AC}\textbf{GNBGauss* (2/4)} & \cellcolor[HTML]{F7C7AC}\textbf{GNB G-cop* (3/4)} & \cellcolor[HTML]{F7C7AC}\textbf{GNB KDE* (2/4)} \\
accuracy & \cellcolor[HTML]{FFFFFF}0.3043 & 0.9826 & 0.9913 & \cellcolor[HTML]{FFFFCC}0.9826 & 0.9870 & \cellcolor[HTML]{FFFFFF}0.9913 & 0.9826 \\
precision & \cellcolor[HTML]{FFFFFF}0.3043 & 0.9828 & 0.9913 & \cellcolor[HTML]{FFFFCC}0.9828 & 0.9870 & \cellcolor[HTML]{FFFFFF}0.9913 & 0.9828 \\
recall & \cellcolor[HTML]{FFFFFF}0.3043 & 0.9826 & 0.9913 & \cellcolor[HTML]{FFFFCC}0.9826 & 0.9870 & \cellcolor[HTML]{FFFFFF}0.9913 & 0.9826 \\
f1-score & \cellcolor[HTML]{FFFFFF}0.2806 & 0.9826 & \textbf{0.9913} & \cellcolor[HTML]{FFFFCC}0.9826 & 0.9870 & \cellcolor[HTML]{FFFFFF}\textbf{0.9913} & 0.9826 \\
\cellcolor[HTML]{DAF2D0}\textbf{Parkinson} & \cellcolor[HTML]{F7C7AC}\textbf{GaussNB} & \cellcolor[HTML]{F7C7AC}\textbf{K.NB} & \cellcolor[HTML]{F7C7AC}\textbf{LR(C=1000)} & \cellcolor[HTML]{F7C7AC}\textbf{LR(C=1.0)} & \cellcolor[HTML]{F7C7AC}\textbf{LR(C=0,01)} & \cellcolor[HTML]{F7C7AC}\textbf{SVM(C=1000)} & \cellcolor[HTML]{F7C7AC}\textbf{SVM(C=1)} \\
accuracy & \cellcolor[HTML]{FFFFCC}0.7267 & 0.7833 & 0.8500 & \cellcolor[HTML]{FFFFFF}0.8600 & 0.7667 & \cellcolor[HTML]{FFFFCC}0.8267 & 0.8700 \\
precision & \cellcolor[HTML]{FFFFCC}0.8230 & 0.8702 & 0.8455 & \cellcolor[HTML]{FFFFFF}0.8610 & 0.5878 & \cellcolor[HTML]{FFFFCC}0.8193 & 0.8838 \\
recall & \cellcolor[HTML]{FFFFCC}0.7267 & 0.7833 & 0.8500 & \cellcolor[HTML]{FFFFFF}0.8600 & 0.7667 & \cellcolor[HTML]{FFFFCC}0.8267 & 0.8700 \\
f1-score & \cellcolor[HTML]{FFFFCC}0.7472 & 0.7994 & 0.8472 & \cellcolor[HTML]{FFFFFF}0.8438 & 0.6654 & \cellcolor[HTML]{FFFFCC}0.8218 & \textbf{0.8515} \\
 & \cellcolor[HTML]{F7C7AC}\textbf{SVM(C=0.01)} & \cellcolor[HTML]{F7C7AC}\textbf{GNBGauss} & \cellcolor[HTML]{F7C7AC}\textbf{GNB G-cop} & \cellcolor[HTML]{F7C7AC}\textbf{GNB KDE} & \cellcolor[HTML]{F7C7AC}\textbf{GNBGauss* (14/15)} & \cellcolor[HTML]{F7C7AC}\textbf{GNB G-cop*} & \cellcolor[HTML]{F7C7AC}\textbf{GNB KDE*(15/15)} \\
accuracy & \cellcolor[HTML]{FFFFFF}0.7667 & 0.7067 & 0.8067 & \cellcolor[HTML]{FFFFCC}0.8067 & 0.7000 & \cellcolor[HTML]{FFFFFF}0.8067 & 0.8067 \\
precision & \cellcolor[HTML]{FFFFFF}0.5878 & 0.7886 & 0.8195 & \cellcolor[HTML]{FFFFCC}0.8350 & 0.7861 & \cellcolor[HTML]{FFFFFF}0.8048 & 0.8350 \\
recall & \cellcolor[HTML]{FFFFFF}0.7667 & 0.7067 & 0.8067 & \cellcolor[HTML]{FFFFCC}0.8067 & 0.7000 & \cellcolor[HTML]{FFFFFF}0.8067 & 0.8067 \\
f1-score & \cellcolor[HTML]{FFFFFF}0.6654 & 0.7277 & 0.8117 & \cellcolor[HTML]{FFFFCC}0.8155 & 0.7218 & \cellcolor[HTML]{FFFFFF}\textbf{0.8057} & \textbf{0.8155}
\end{tabular}
\end{table}
\end{landscape}
\newpage
\begin{landscape}
% Beamer presentation requires \usepackage{colortbl} instead of % Please add the following required packages to your document preamble:
\begin{table}[]
\caption{The table contains the mean accuracy scores (rows) for different methods (columns) over the same 10 random splits on the Wine and Wine quality data. We highlighted in yellow when all features were used, and by star our reduced models.}
\label{tab:Wine}
\begin{tabular}{cc
>{\columncolor[HTML]{FFFFCC}}c 
>{\columncolor[HTML]{FFFFCC}}c ccc
>{\columncolor[HTML]{FFFFFF}}c }
\cellcolor[HTML]{DAF2D0}\textbf{Wine} & \cellcolor[HTML]{F7C7AC}GaussNB & \cellcolor[HTML]{F7C7AC}K.NB & \cellcolor[HTML]{F7C7AC}LR(C=1000) & \cellcolor[HTML]{F7C7AC}LR(C=1.0) & \cellcolor[HTML]{F7C7AC}LR(C=0,01) & \cellcolor[HTML]{F7C7AC}SVM(C=1000) & \cellcolor[HTML]{F7C7AC}SVM(C=1) \\
accuracy & \cellcolor[HTML]{FFFFCC}0.9926 & 0.9222 & 0.9741 & \cellcolor[HTML]{FFFFFF}0.9778 & \cellcolor[HTML]{FFFFFF}0.5630 & \cellcolor[HTML]{FFFFCC}0.9667 & 0.9778 \\
precision & \cellcolor[HTML]{FFFFCC}0.9928 & 0.9318 & 0.9740 & \cellcolor[HTML]{FFFFFF}0.9786 & \cellcolor[HTML]{FFFFFF}0.5299 & \cellcolor[HTML]{FFFFCC}0.9672 & 0.9795 \\
recall & \cellcolor[HTML]{FFFFCC}0.9926 & 0.9222 & 0.9741 & \cellcolor[HTML]{FFFFFF}0.9778 & \cellcolor[HTML]{FFFFFF}0.5630 & \cellcolor[HTML]{FFFFCC}0.9667 & 0.9778 \\
f1-score & \cellcolor[HTML]{FFFFCC}0.9926 & 0.9207 & 0.9740 & \cellcolor[HTML]{FFFFFF}0.9777 & \cellcolor[HTML]{FFFFFF}0.4773 & \cellcolor[HTML]{FFFFCC}0.9666 & 0.9779 \\
 & \cellcolor[HTML]{F7C7AC}SVM(C=0.01) & \cellcolor[HTML]{F7C7AC}GNBGauss & \cellcolor[HTML]{F7C7AC}GNB G-cop & \cellcolor[HTML]{F7C7AC}GNB KDE & \cellcolor[HTML]{F7C7AC}GNBGauss* (12/12) & \cellcolor[HTML]{F7C7AC}GNB G-cop* (12/12) & \cellcolor[HTML]{F7C7AC}GNB KDE* (12/12) \\
accuracy & \cellcolor[HTML]{FFFFFF}0.4074 & 1.0000 & 1.0000 & \cellcolor[HTML]{FFFFCC}0.9889 & 1.0000 & 1.0000 & 0.9889 \\
precision & \cellcolor[HTML]{FFFFFF}0.1660 & 1.0000 & 1.0000 & \cellcolor[HTML]{FFFFCC}0.9893 & 1.0000 & 1.0000 & 0.9893 \\
recall & \cellcolor[HTML]{FFFFFF}0.4074 & 1.0000 & 1.0000 & \cellcolor[HTML]{FFFFCC}0.9889 & 1.0000 & 1.0000 & 0.9889 \\
f1-score & \cellcolor[HTML]{FFFFFF}0.2359 & \textbf{1.0000} & \textbf{1.0000} & \cellcolor[HTML]{FFFFCC}0.9889 & \textbf{1.0000} & \textbf{1.0000} & 0.9889 \\
\cellcolor[HTML]{DAF2D0}\textbf{Wine quality} & \cellcolor[HTML]{F7C7AC}GaussNB & \cellcolor[HTML]{F7C7AC}K.NB & \cellcolor[HTML]{F7C7AC}LR(C=1000) & \cellcolor[HTML]{F7C7AC}LR(C=1.0) & \cellcolor[HTML]{F7C7AC}LR(C=0,01) & \cellcolor[HTML]{F7C7AC}SVM(C=1000) & \cellcolor[HTML]{F7C7AC}SVM(C=1) \\
accuracy & \cellcolor[HTML]{FFFFCC}0.7443 & 0.6788 & 0.8185 & \cellcolor[HTML]{FFFFFF}0.8204 & \cellcolor[HTML]{FFFFFF}0.8031 & \cellcolor[HTML]{FFFFCC}0.8031 & 0.8031 \\
precision & \cellcolor[HTML]{FFFFCC}0.4040 & 0.3583 & 0.5884 & \cellcolor[HTML]{FFFFFF}0.6234 & \cellcolor[HTML]{FFFFFF}0.0000 & \cellcolor[HTML]{FFFFCC}0.0000 & 0.0000 \\
recall & \cellcolor[HTML]{FFFFCC}0.6281 & 0.7984 & 0.2599 & \cellcolor[HTML]{FFFFFF}0.2224 & \cellcolor[HTML]{FFFFFF}0.0000 & \cellcolor[HTML]{FFFFCC}0.0000 & 0.0000 \\
f1-score & \cellcolor[HTML]{FFFFCC}0.4917 & 0.4947 & 0.3605 & \cellcolor[HTML]{FFFFFF}0.3278 & \cellcolor[HTML]{FFFFFF}0.0000 & \cellcolor[HTML]{FFFFCC}0.0000 & 0.0000 \\
 & \cellcolor[HTML]{F7C7AC}SVM(C=0.01) & \cellcolor[HTML]{F7C7AC}GNBGauss & \cellcolor[HTML]{F7C7AC}GNB G-cop & \cellcolor[HTML]{F7C7AC}GNB KDE & \cellcolor[HTML]{F7C7AC}GNBGauss* (12/12) & \cellcolor[HTML]{F7C7AC}GNB G-cop* (10/10) & \cellcolor[HTML]{F7C7AC}GNB KDE* (12/12) \\
accuracy & \cellcolor[HTML]{FFFFFF}0.8031 & 0.7753 & 0.8009 & \cellcolor[HTML]{FFFFCC}0.8181 & \cellcolor[HTML]{FFFFFF}0.7753 & \cellcolor[HTML]{FFFFFF}0.8009 & 0.8181 \\
precision & \cellcolor[HTML]{FFFFFF}0.0000 & 0.4502 & 0.4942 & \cellcolor[HTML]{FFFFCC}0.5369 & \cellcolor[HTML]{FFFFFF}0.4502 & \cellcolor[HTML]{FFFFFF}0.4942 & 0.5369 \\
recall & \cellcolor[HTML]{FFFFFF}0.0000 & 0.6375 & 0.4688 & \cellcolor[HTML]{FFFFCC}0.5526 & \cellcolor[HTML]{FFFFFF}0.6375 & \cellcolor[HTML]{FFFFFF}0.4688 & 0.5526 \\
f1-score & \cellcolor[HTML]{FFFFFF}0.0000 & 0.5277 & 0.4812 & \cellcolor[HTML]{FFFFCC}\textbf{0.5447} & \cellcolor[HTML]{FFFFFF}0.5277 & \cellcolor[HTML]{FFFFFF}0.4812 & \textbf{0.5447}
\end{tabular}
\end{table}
\end{landscape}
\section{Conclusion}
\label{sec:Conclusions}
This paper discusses the classification based on fitting GNB structures to data when the explanatory variables are continuous. We introduce the theoretical background to solve this problem under different assumption on the joint probability distribution between the explanatory variables: joint a Gaussian distribution; the more general case when the uni-variate marginals may have various distributions, and the dependence is characterized by a Gaussian copula; and an even more general case when we have no restriction on the joint density (copula or marginals). 
We proved, that the KL divergence depends on the GNB structure only through the pair copulas associated to the bi-variate marginals.
We give under each of the assumptions a classification method. Moreover we give also a feature/cluster reduction methodology together with a classification based on the concept of GNB forest.
The GNB structure learning based on data was also discussed from combinatorial point of view. We proved that the GNB structures may be assigned to the basis of a matroid. This ensures the existence of a greedy algorithm with an optimality guarantee, in the sense of minimizing the KL divergence. For finding such greedy algorithm we defined a weighted complete graph on the set of indices, and proved that applying to it the Kruskal's algorithm we obtain the optimal GNB structure.
We introduced a more general structure allowing more complex dependence structure between the attributes ( $k\geq3$ interconnected features) called $k$-GNB models, and discussed from mathematical point of view the Gaussian case. For $k\geq3$ there are no guarantees for finding the optimal structure even for the Gaussian case.

%We proposed greedy algorithms for structure discovery. 

\textit{Advantages}: A significant strength of the new algorithms, in contrast to black-box algorithms, is their explainability. Therefore the GNB structures can be considered an interpretable machine learning tool. 
A huge advantage of using the simple GNB approach is that all marginal probability distributions used for classification are of order 3. Even in the more general cases, when $k>3$ it is an advantage to chose small values for $k$; Low dimensional marginals can be more trustfully approximated from data, even if the data is not very large. Therefore, our algorithms also work well in cases with many features and relatively small datasets.
We observed that via model reduction the accuracy of the classification  is typically improved. The output graphs are usable for experts in deciding on how many clusters should be considered. This way, they may intervene in the classification loop.
All the proposed algorithms for GNB were tested and compared under the same conditions. It turned out that our algorithms are competitive, in many cases even better than the popular Glass-box models.
Our algorithms were implemented in Python, a demo repository can be found at  
\newline  \url{https://github.com/AbrahamPapp/generalized-naive-bayes-demo}.
We will publish also the package developed by the first author.

\textit{Limitations}: The Gaussian classification method should be used with precaution. The copula based methods need in the classification task the kernel densities fitted to the training data. Using the default parameters may be suboptimal. We recommend fine-tuning the parameters involved with respect to the given problem.
Another drawback may occur even in the cases when we have nice guarantee on the optimality of structure on the training data, this gives no guarantee on its efficiency on the test data. However the method of reduction may overcome this issue.

\textit{Future work:}
We plan to accelerate the algorithm and discuss it for higher values of $k \geq 3$. Additionally, we intend to release a Python package containing these algorithms.

We anticipate that this work will be advantageous for classification problems in fields where acquiring large sample data is very expensive or time-consuming. 

\vspace{-3mm}
\section*{Acknowledgements}
\noindent The research reported in this paper was partially supported by project no. BME-NVA-02, implemented with the support provided by the Ministry of Innovation and Technology of Hungary from the National Research, Development and Innovation Fund, financed under the TKP2021 funding scheme.
%\section*{Declaration of competing interest}
%\noindent The authors declare that they have no known competing financial interests or personal relationships that could have appeared to influence the work reported in this paper.

%\bibliography{cas-refs.bib}
\bibliographystyle{elsarticle-num} 

\newpage

\appendix

\section{cherry junction tree of order $k$}
\label{Appendix:cherry tree}
We call the structure assigned to a cherry tree graph of order $k$ a \textbf{cherry-junction tree of order \textbf{k}}, and define it as follows:
\begin{enumerate}
\item  The set of vertices in each $k$th order maximum clique is called a \textbf{cluster}, and contains its $k$ vertices. 
\item Two $k$-element clusters are \textbf{connected} if the following two conditions are fulfilled:
\begin{itemize}
\item the clusters share $k-1$ elements
\item if a set of $k-1$ elements is contained by $m$ clusters, these clusters will be connected tree-like by $m-1$ edges.
\end{itemize}
\item The $k-1$-element set given by the intersection of two connected clusters is called a \textbf{separator}. The number of clusters containing it is the \textbf{multiplicity of the separator}.
\end{enumerate}
A cherry-junction tree is characterized by the set of indices $V$, the set of clusters $\mathcal{C}$, the set of separators  $\mathcal{S}$ and a dictionary of the multiplicities of the separators $\mathcal{M}$.

The associated probability distribution has the formula:
\begin{equation*}
\label{eq:Cherry_tree_original}
p_{GNB}( \mathbf{x}) =\frac{\prod\limits_{C\in \mathcal{C}}p\left( \mathbf{x}_{C}\right) }{\prod\limits_{S\in \mathcal{S}}p\left( \mathbf{x}_{S}\right)
^{v_{S}-1}}
\end{equation*}

\section{Information content of continuous and discrete random vector}
\label{Appendix:Information content}
\begin{proof}
The entropy of the random vector $(\mathbf{X},Y)=(X_1,...,X_d,Y)$ is the following:
\[
H(\mathbf{X},Y)=-\sum\limits_{y_{k}\in \mathcal{Y}}\int_{\mathbf{\chi}} p(\mathbf{x}%
,y_{k})\log p(\mathbf{x},y_{k})dx
\]

We express $p(\mathbf{x},y)$ from the conditional probability distribution
formula:%
\[
p(\mathbf{x}|Y)=\frac{p(\mathbf{x},y)}{p(Y)}\implies p(\mathbf{x},y)=p(Y)p(\mathbf{x}|Y)
\]%
and substitute it in the formula of the joint entropy:

\begin{eqnarray*}
H(\mathbf{X},Y) &=&-\sum\limits_{y_{k}\in \mathcal{Y}}\int_{\mathbf{\chi}} p(y_{k})p(\mathbf{%
x}|y_{k})\log [p(y_{k})p(\mathbf{x}|y_{k})]dx= \\
&=&-\sum\limits_{y_{k}\in \mathcal{Y}}p(y_{k})\int_{\mathbf{\chi}}  p(\mathbf{x}|y_{k})\log
[p(y_{k})p(\mathbf{x}|y_{k})]dx= \\
&=&-\sum\limits_{y_{k}\in \mathcal{Y}}\left[ p(y_{k})\int_{\mathbf{\chi}}  p(\mathbf{x}%
|y_{k})\log p(y_{k})+p(\mathbf{x}|y_{k})\log p(\mathbf{x}|y_{k})\right]dx=
\\
&=&-\sum\limits_{y_{k}\in \mathcal{Y}}\left( p(y_{k})\log p(y_{k})\underset{1%
}{\underbrace{\int_{\mathbf{\chi}}  p(\mathbf{x}|y_{k})dx}}+p(y_{k})\int_{\mathbf{\chi}}  p(\mathbf{x}%
|y_{k})\log p(\mathbf{x}|y_{k})dx\right) = \\
&=&-\sum\limits_{y_{k}\in \mathcal{Y}}p(y_{k})\log p(y_{k}) 
+\sum\limits_{y_{k}\in \mathcal{Y}}p(y_{k})H(\mathbf{X}%
|y_{k})=\\
&=&H(Y)+\sum\limits_{y_{k}\in \mathcal{Y}}p(y_{k})H(\mathbf{X}|y_{k})
\end{eqnarray*}
By using the formula of joint entropy, we can express the formula of information content as follows:
\begin{eqnarray*}
I(\mathbf{X},Y) &=&\sum\limits_{i=1}^{d}H(X_{i})+H(Y)-H(\mathbf{X},Y)= \notag\\
&=&\sum\limits_{i=1}^{d}H(X_{i})+H(Y)-\left( H(Y)+\sum\limits_{y_{k}\in 
\mathcal{Y}}p(y_{k})H(\mathbf{X}|y_{k})\right)  \notag\\
&=&\sum\limits_{i=1}^{d}H(X_{i})-\sum\limits_{y_{k}\in \mathcal{Y}}p(y_{k})H(\mathbf{X}|y_{k})
\end{eqnarray*}
\end{proof}

\section{Mutual information Gauss and discrete random variables}
\label{Appendix: I(X,Y)_Gauss_proof}
\begin{eqnarray}
I(X,Y) &=&H(X)-\sum\limits_{y_{k}\in \mathcal{Y}}p(y_{k})H(X|y_{k})= \notag \\
&=&\frac{1}{2}\log (2\pi e\sigma_{X} ^{2})-\sum\limits_{y_{k}\in \mathcal{Y}%
}p(y_{k})\left[ \frac{1}{2}\log (2\pi e\sigma _{X|y_{k}}^{2})\right] = \notag \\
&=&\frac{1}{2}\left( \log 2\pi +1+\log \sigma_{X} ^{2}\right) +\frac{-1}{2}%
\sum\limits_{y_{k}\in \mathcal{Y}}p(y_{k})\left[ \log (2\pi )+1+\log \sigma
_{X|y_{k}}^{2}\right] = \notag \\
&=&\frac{1}{2}\left[ \left( \log 2\pi +1+\log\sigma_{X} ^{2}\right) -\log (2\pi
)-1-\sum\limits_{y_{k}\in \mathcal{Y}}p(y_{k})\log \sigma _{X|y_{k}}^{2}\right]
= \notag\\ 
&=&\frac{1}{2}\left(\log \sigma^2_{X} -\sum\limits_{y_{k}\in \mathcal{Y}%
}p(y_{k})\log \sigma ^2
_{X|y_{k}}\right)
\end{eqnarray}
\section{Evaluation metrics}
\label{Appendix:Accuracy}
In binary classification, beyond accuracy, precision and recall are often used metrics to evaluate the models. Let the labels we want to predict be `positive' and `negative', then we can define the \emph{confusion matrix}, in which the columns show the actual values and the rows show the predicted values:
\begin{table}[h!]
\begin{tabular}{cc|c|c|}
\cline{3-4}
\multicolumn{2}{c}{\multirow{2}{*}{}} & \multicolumn{2}{|c|}{Actual} \\ \cline{3-4} 
\multicolumn{2}{c|}{} & Positive & Negative \\ \hline
\multicolumn{1}{|c|}{\multirow{2}{*}{Predicted}} & Positive & \#\{\textbf{True Positive}\} & \#\{\textbf{False Positive}\} \\ \cline{2-4} 
\multicolumn{1}{|c|}{} & Negative & \#\{\textbf{False Negative}\} & \#\{\textbf{True Negative}\} \\ \hline
\end{tabular}
\end{table}

There are four possible outcomes according to the actual and the predicted label. The elements of the matrix show the occurrence of each outcome. Based on these accuracy, precision and recall can be defined as follows:

\begin{equation*}
    \textrm{Accuracy} = \frac{\#\{\textrm{True Positive}\}+\#\{\textrm{True Negative}\}}{\#\{Total\}}
\end{equation*}

\begin{equation*}
    \textrm{Precision} = \frac{\#\{\textrm{True Positive}\}}{\#\{\textrm{True Positive}\}+\#\{\textrm{False Positive}\}}
\end{equation*}

\begin{equation*}
    \textrm{Recall} = \frac{\#\{\textrm{True Positive}\}}{\#\{\textrm{True Positive}\}+\#\{\textrm{False Negative}\}}
\end{equation*}

\begin{equation*}
    \textrm{f1 score} = 2\times \frac{\textrm{Precision}\times \textrm{Recall}}{\textrm{Precision}+\textrm{Recall}}
\end{equation*}

\section{Model reduction based on graphs}
\label{Appendix:model_reduction_graphs} 

\begin{figure}[H]
    \centering
    \includegraphics[width=1\textwidth]{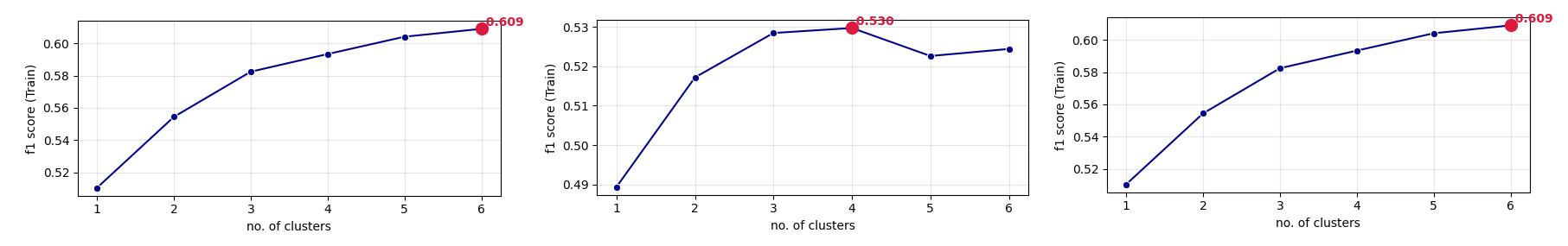}
    \caption{f1 score on Abalone training data as a function of the number of clusters (right: Gaussian, middle Gauss copula, left: general)} 
    \label{Abalone_red}
\end{figure}

\begin{figure}[H]
    \centering
    \includegraphics[width=1\textwidth]{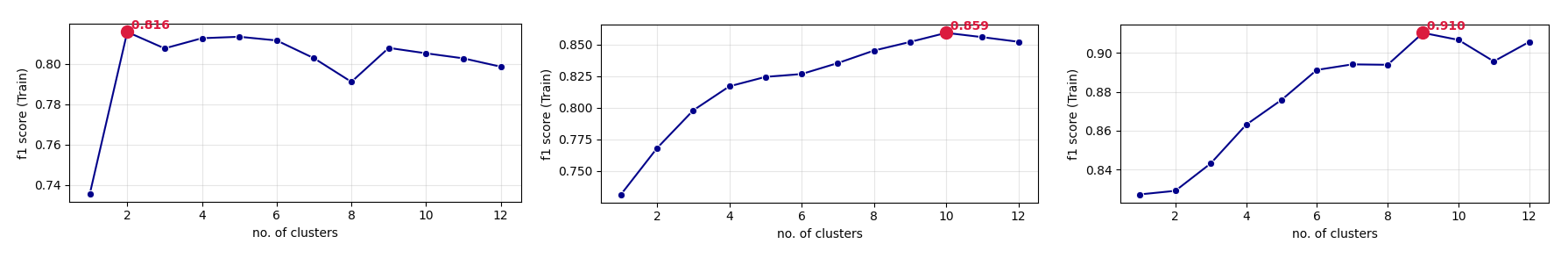}
    \caption{f1 score on Body fat training data as a function of the number of clusters (right: Gaussian, middle Gauss copula, left: general)} 
    \label{Body fat_red}
\end{figure}

\begin{figure}[H]
    \centering
    \includegraphics[width=1\textwidth]{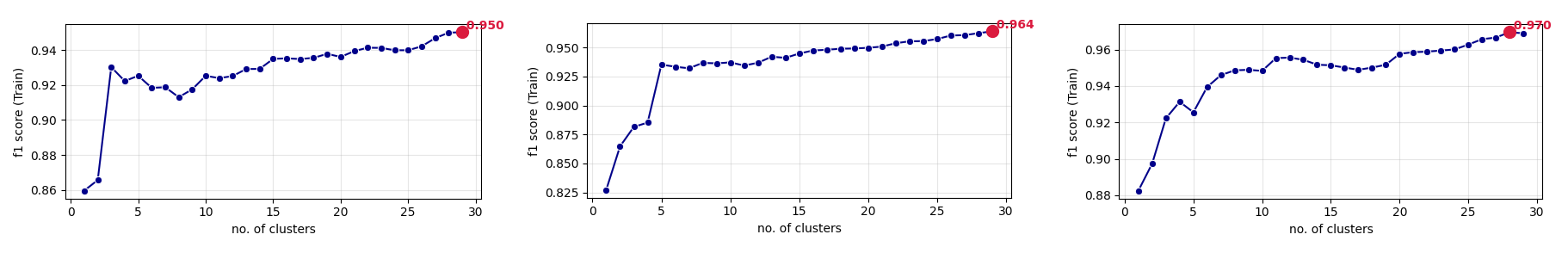}
    \caption{f1 score on WBCD training data as a function of the number of clusters (right: Gaussian, middle Gauss copula, left: general)} 
    \label{WBCD_red}
\end{figure}

\begin{figure}[H]
    \centering
    \includegraphics[width=1\textwidth]{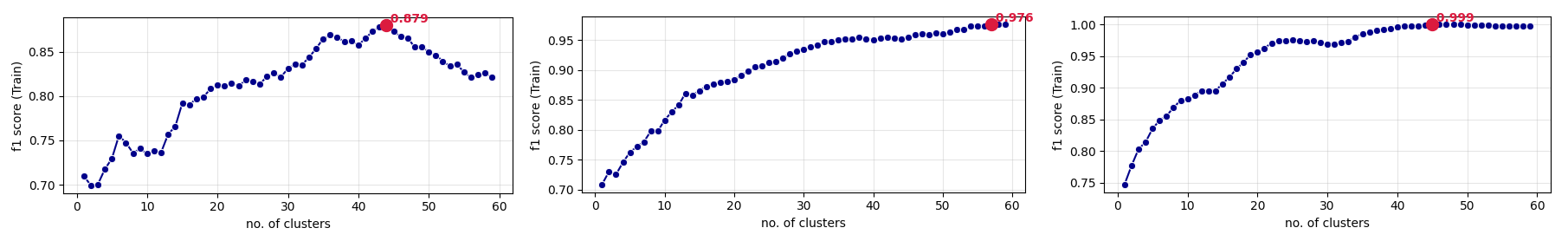}
    \caption{f1 score on Connect training data as a function of the number of clusters (right: Gaussian, middle Gauss copula, left: general)} 
    \label{Connect_red}
\end{figure}

\begin{figure}[H]
    \centering
    \includegraphics[width=1\textwidth]{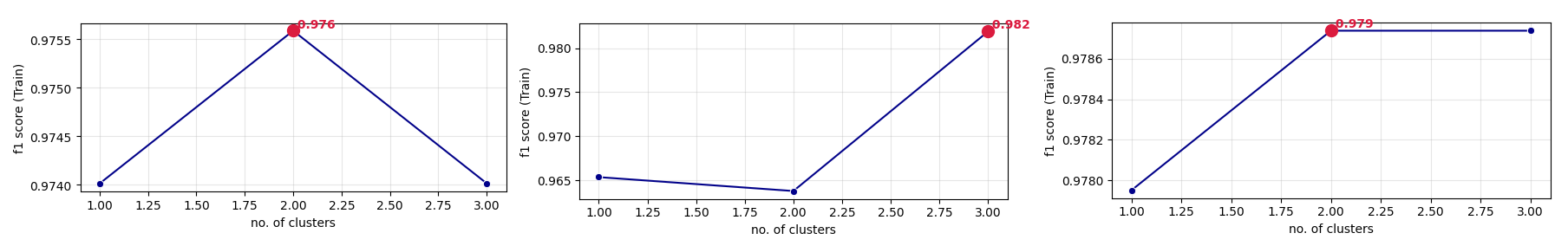}
    \caption{f1 score on Iris training data as a function of the number of clusters (right: Gaussian, middle Gauss copula, left: general)} 
    \label{Iris_red}
\end{figure}

\begin{figure}[H]
    \centering
    \includegraphics[width=1\textwidth]{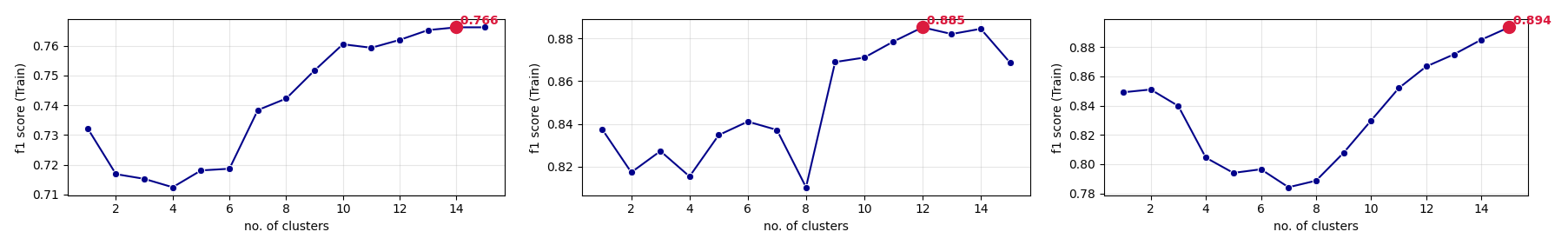}
    \caption{f1 score on Parkinson training data as a function of the number of clusters (right: Gaussian, middle Gauss copula, left: general)} 
    \label{Parkinson_red}
\end{figure}

\begin{figure}[H]
    \centering
    \includegraphics[width=1\textwidth]{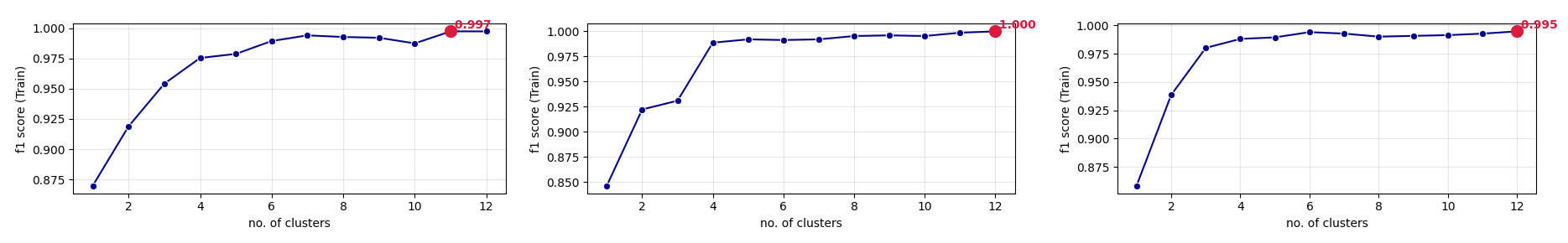}
    \caption{f1 score on Wine training data as a function of the number of clusters (right: Gaussian, middle Gauss copula, left: general)} 
    \label{Wine_red}
\end{figure}
\begin{figure}[H]
    \centering
    \includegraphics[width=1\textwidth]{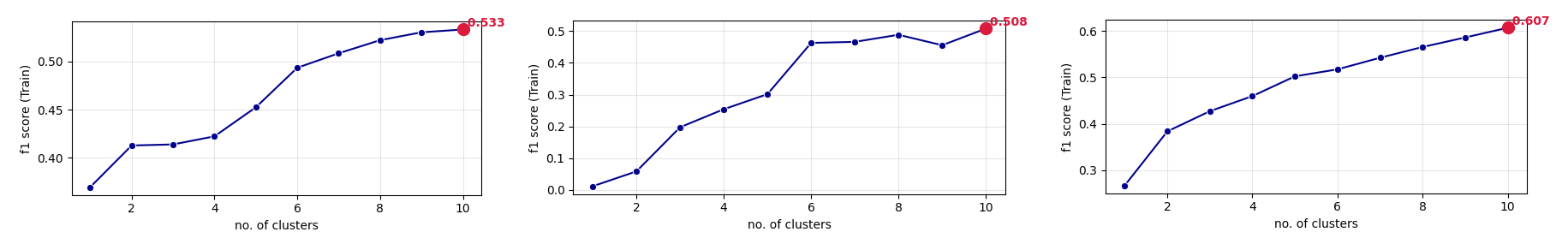}
    \caption{f1 score on Wine quality training data as a function of the number of clusters (right: Gaussian, middle Gauss copula, left: general)} 
    \label{Red wine quality}
\end{figure}

\section{Pair diagnostic graphs}
\label{Apendix_diagnostics based on pairs}
Illustration of the diagnostics for the Iris Data for the feature pairs used in the GNB structure under different assumptions/methods applied.
\begin{figure}[H]
    \centering
    \includegraphics[width=1\textwidth]{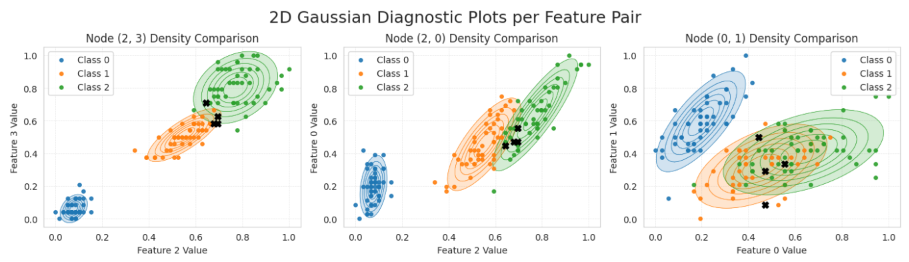}
    \caption{The good (colored) and the wrongly  classified points (black) in the chosen Gaussian marginal pairs from the Gauss GNB structure on the Iris \cite{iris_53} test set.}
    \label{Gauss_Iris}
\end{figure}

\begin{figure}[H]
    \centering
    \includegraphics[width=1\textwidth]{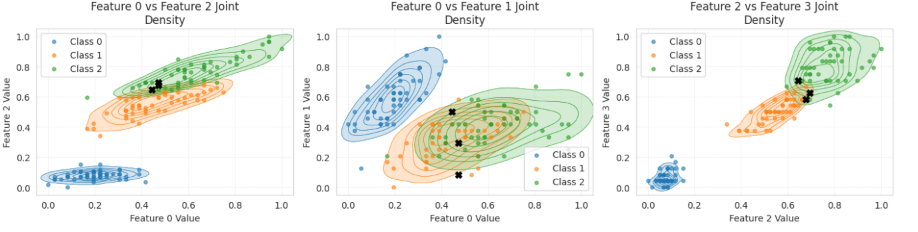}
    \caption{In the chosen pairs belonging to the GNB Gauss copula structure with kde uni-variate marginals, the good classified points are colored with respect to their class; the wrongly classified points are represented by black points, on the Iris \cite{iris_53} test set.}
    \label{Gauss copula Iris}
\end{figure}

\begin{figure}[H]
    \centering
    \includegraphics[width=1\textwidth]{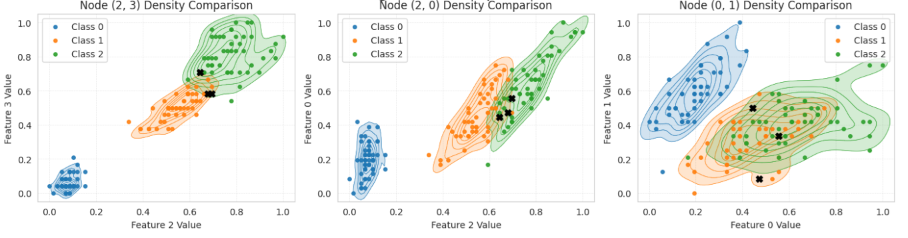}
    \caption{In the chosen pairs belonging to the KDE GNB structure (with kde bi-variate distribution) the good classified points are colored with respect to their class; the wrongly classified points are represented by black points, on the Iris \cite{iris_53} test set.}
    \label{KDE Iris}
\end{figure}

\begin{figure}[H]
    \centering
    \includegraphics[width=1\textwidth]{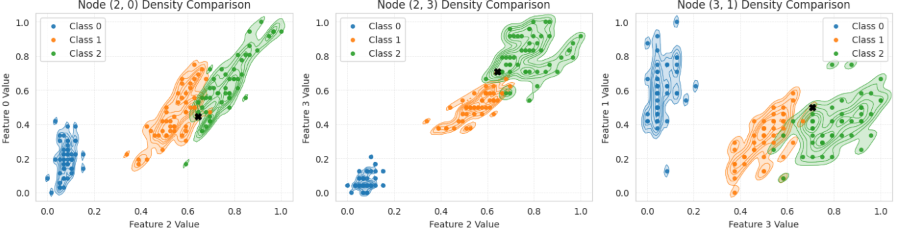}
    \caption{In the chosen pairs belonging to the KDE GNB structure with manually tuned parameters. The classified points are colored with respect to their class; the wrongly classified points are represented by black points, on the Iris \cite{iris_53} test set.}
    \label{KDE tuned Iris}
\end{figure}

\end{document}